\documentclass[nohyperref]{article}

\usepackage{hyperref}

\usepackage[citestyle=authoryear,
natbib=true,backend=biber,doi=false,isbn=false,url=false,maxcitenames=2]{biblatex}
\renewbibmacro{in:}{\ifentrytype{article}
	{}
	{\bibstring{in}\printunit{\intitlepunct}}
}

\usepackage[accepted]{icml2023}

\usepackage[textsize=tiny]{todonotes}

\usepackage[utf8]{inputenc} 
\usepackage[T1]{fontenc}    
\usepackage{hyperref}       
\usepackage{url}            
\usepackage{booktabs}       
\usepackage{amsfonts}       
\usepackage{nicefrac}       
\usepackage{microtype}      
\usepackage{xcolor}         

\usepackage{hyperref} 
\usepackage{url} 
\usepackage{enumitem} 
\usepackage{outlines} 

\usepackage{amsmath,amsthm,amsfonts,amssymb} 
\usepackage{bbm} 
\usepackage{bm} 
\usepackage{dsfont} 
\usepackage{mathtools} 
\usepackage{braket} 
\usepackage{thmtools} 
\usepackage{thm-restate} 

\usepackage{listings} 

\usepackage{csquotes} 
\usepackage[capitalize,noabbrev]{cleveref} 
\usepackage{nameref} 

\usepackage{float} 
\usepackage{graphicx} 
\usepackage[export]{adjustbox} 
\usepackage{svg} 
\usepackage{tikz} 
\usepackage{caption} 
\usepackage{subcaption} 

\usepackage{array,tabularx} 
\usepackage{booktabs} 
\usepackage{longtable} 
\usepackage{multirow}
\usepackage{multicol} 
\usepackage{makecell}

\usepackage{soul} 
\usepackage{indentfirst} 
\usepackage[htt]{hyphenat} 

\usepackage[toc,page,header]{appendix}
\usepackage{minitoc}

\usepackage{etoolbox} 

\crefname{equation}{}{}

\newtheorem{theorem}{Theorem}
\theoremstyle{definition}
\newtheorem{proposition}{Proposition}
\theoremstyle{definition}

\theoremstyle{definition}

\theoremstyle{remark}

\makeatletter
\newcommand{\includesvgfullpath}[2][\textwidth]{%
	\filename@parse{#2}%
	\includesvg[inkscapepath=svg-inkscape/\filename@area,width=#1]{#2}%
}
\makeatother

\DeclareMathOperator*{\argmin}{arg\,min}

\newcommand{\R}{\mathbb{R}}

\newcommand{\D}{\mathcal{D}}

\newcommand{\EE}[2]{\mathbb{E}_{#1\!\!}\left[#2\right]}
\def\E#1{\EE{\,}{#1}}

\DeclarePairedDelimiterX{\infdivx}[2]{(}{)}{%
	#1\;\delimsize\|\;#2%
}

\newcommand{\indicator}{\mathbbm{1}}

\def\defeq{\doteq}

\let\mc\mathcal                                             
                                                 % ...

\def\ceil#1{\lceil #1 \rceil}

\DeclareMathAlphabet{\mathsfit}{\encodingdefault}{\sfdefault}{m}{sl}
\SetMathAlphabet{\mathsfit}{bold}{\encodingdefault}{\sfdefault}{bx}{n}

\newtoggle{revision}
\togglefalse{revision} 

\newcommand{\rev}[1]{%
	\iftoggle{revision}{%
		\textcolor{blue}{#1}%
	}{%
		#1%
	}%
}

\icmltitlerunning{Probabilistic Calibration in Neural Network Regression}

\begin{document}
	
	\twocolumn[
	\icmltitle{A Large-Scale Study of Probabilistic Calibration \\ in Neural Network Regression}
	
	
	
	
	\begin{icmlauthorlist}
		\icmlauthor{Victor Dheur}{umons}
		\icmlauthor{Souhaib Ben Taieb}{umons}
	\end{icmlauthorlist}
	
	\icmlaffiliation{umons}{Department of Computer Science, University of Mons, Mons, Belgium}
	
	\icmlcorrespondingauthor{Victor Dheur}{victor.dheur@umons.ac.be}
	
	\icmlkeywords{Calibration, Regression, Study, Regularization, Post-hoc, Neural Networks, Deep Learning, Machine Learning, ICML}
	
	\vskip 0.3in
	]
	
	
	
	\printAffiliationsAndNotice{}  
	

	
\begin{abstract}
Accurate probabilistic predictions are essential for optimal decision making. While neural network miscalibration has been studied primarily in classification, we investigate this in the less-explored domain of regression. We conduct the largest empirical study to date to assess the probabilistic calibration of neural networks. We also analyze the performance of recalibration, conformal, and regularization methods to enhance probabilistic calibration. Additionally, we introduce novel differentiable recalibration and regularization methods, uncovering new insights into their effectiveness. Our findings reveal that regularization methods offer a favorable tradeoff between calibration and sharpness. Post-hoc methods exhibit superior probabilistic calibration, which we attribute to the finite-sample coverage guarantee of conformal prediction. Furthermore, we demonstrate that quantile recalibration can be considered as a specific case of conformal prediction. Our study is fully reproducible and implemented in a common code base for fair comparisons.
\end{abstract}
	
	\section{Introduction}
	
	Neural network predictions affect critical decisions in many applications, including medical diagnostics and autonomous driving \citep{Gulshan2016-kw,Guizilini2020-tr}.
	However, effective decision making often requires accurate probabilistic predictions \citep{Gawlikowski2021-ty,Abdar2021-zq}.
	For example, consider a probabilistic regression model that produces 90\% prediction intervals. An important property would be that 90\% of these prediction intervals contain the realizations.
	
	
	For models that output a predictive distribution, \textit{probabilistic calibration} is an important property that states that all quantiles must be calibrated\rev{, i.e., the frequency of realizations below these quantiles must match the corresponding quantile level}. Additionally, predictive distributions should be \rev{sufficiently} sharp (i.e., concentrated around the realizations) \rev{and leverage the information in the inputs}.
	
	
	

	In the classification setting, \citet{Guo2017-ow} found that common neural architectures trained on image and text data were miscalibrated, sparking increased interest in neural network calibration. In a follow-up study, \citet{Minderer2021-xw} showed that more recent neural architectures demonstrate improved calibration. However, there has been less research on calibration for neural probabilistic regression models compared to classification. Therefore, it remains uncertain whether the same results apply to the regression setting. This paper addresses this gap by conducting a comprehensive study on probabilistic calibration for regression using tabular data. We explore various calibration methods, including quantile recalibration \citep{Kuleshov2018-tb} and conformalized quantile regression \citep{Romano2019-kp}. We also consider regularization methods, which have been shown to perform well in the classification setting \citep{Karandikar2021-ml,Popordanoska2022-jz,Yoon2023-ds}.

	We make the following main contributions:
	\begin{enumerate}
		\item We conduct the largest empirical study to date on probabilistic calibration of neural regression models using 57 tabular datasets (\cref{sec:empirical_study,sec:comparative_study}). We consider multiple state-of-the-art calibration methods (Section 5), including post-hoc recalibration, conformal prediction, and regularization methods, with various scoring rules and predictive models. 
		
		\item Building on quantile recalibration, we propose a new differentiable calibration map using kernel density estimation, which provides improved negative log-likelihood compared to baselines. We also introduce two new regularization objectives based on the probabilistic calibration error (\cref{sec:methods}).
		
		
		\item We show that quantile recalibration is a special case of conformal prediction, providing an explanation for their superior performance in terms of probabilistic calibration (\cref{sec:comparative_study}).
	\end{enumerate}

	\section{Background}
	\label{sec:background}
	
	We consider a univariate regression problem where the target variable $Y \in \mc{Y}$ depends on an input variable $X \in \mc{X}$, with $\mc{Y} = \mathbb{R}$ representing the target space and $\mc{X}$ representing the input space. Our objective is to approximate the conditional distribution $P_{Y \mid X}$ using training data $\D = \Set{(X_i, Y_i)}_{i=1}^N$ where $(X_i, Y_i) \overset{\text{i.i.d.}}{\sim} P \equiv P_X \times P_{Y \mid X}$.
	
	A probabilistic predictor $F_\theta: \mc{X} \to \mc{F}$ is a function parametrized by $\theta \in \Theta$ that maps an input $x \in \mc{X}$ to a predictive cumulative distribution function (CDF) $F_\theta(\cdot \mid x)$ in the space $\mc{F}$ of distributions over $\R$.
	Additionally, given $x \in \mc{X}$, we denote the predictive quantile function (QF) by $Q_\theta(\cdot \mid x)$, and probability density function (PDF) by $f_\theta(\cdot \mid x)$.
	Similarly, the marginal CDF, QF, or PDF of a random variable $R$ is denoted by $F_R$, $Q_R$, or $f_R$, respectively.
	
	\paragraph{Probabilistic calibration.}
	
	Given an input $x \in \mc{X}$, the model $F_\theta$ is ideal if it precisely matches the conditional distribution $P_{Y \mid X}$.
	However, learning the ideal model based on finite data is not possible without additional (strong) assumptions \citep{Foygel_Barber2021-ig}. To avoid additional assumptions, we can instead enforce certain desirable properties that are attainable in practice and that a good or ideal forecaster should exhibit. One such property is probabilistic calibration.

	
	

	Let $Z = F_\theta(Y \mid X) \in [0, 1]$ denote the probability integral transform (PIT) of $Y$ conditional on $X$.
	The model $F_\theta$ is \emph{probabilistically calibrated} (also known as PIT-calibrated) if $\forall \alpha \in [0, 1]$,
	\begin{equation}
		\label{eq:probabilistic_calibration}
		F_Z(\alpha) \defeq \text{Pr}(Z \leq \alpha) = \alpha.
	\end{equation}
	Let $U \in [0, 1]$ be a uniform random variable independent of $Z$.
	The left and right hand sides of \cref{eq:probabilistic_calibration} can be interpreted as the CDF of $Z$ and $U$, respectively, as a function of $\alpha$.
	This shows that the uniformity of the PIT is equivalent to probabilistic calibration \citep{Dawid1984-gn}.
	
	
	Since the ideal forecaster is probabilistically calibrated, we can require this property from any competent forecaster. However, probabilistic calibration, though necessary, is not sufficient for making accurate probabilistic predictions. Additionally, as discussed by \citet{Gneiting2021-vu}, probabilistic calibration primarily addresses unconditional aspects of predictive performance and is implied by more robust conditional notions of calibration, such as auto-calibration.

	\paragraph{Probabilistic calibration error.}
	
	The most common approach for evaluating probabilistic calibration is to consider distances of the form $\int_0^1 |F_Z(\alpha) - F_U(\alpha)|^p d\alpha$ where $p > 0$.
	The particular cases of $p = 1$ and $p = 2$ are known as the 1-Wasserstein distance and Cramér-von Mises distance, respectively.
	We denote the empirical CDF of the PIT as $\hat{F}_Z(\alpha) = \frac{1}{N} \sum_{i=1}^N \indicator(Z_i \leq \alpha)$ where $Z_i = F_\theta(Y_i \mid X_i)$ are PIT realizations. A common approach to assess probabilistic calibration \rev{using Monte Carlo estimation} is to evaluate it at equidistant values $\alpha_1 < \dots < \alpha_M$ as follows:
	\begin{equation}
		\label{eq:UQCE}
		\text{PCE}_p(F_\theta, \D) = \frac{1}{M} \sum_{j=1}^M \left|\alpha_j - \hat{F}_Z(\alpha_j) \right|^p.
	\end{equation}
	This metric has been previously employed in literature such as \citet{Zhao2020-ze, Zhou2021-hx} with $p = 1$, and \citet{Kuleshov2018-tb, Utpala2020-nw} with $p = 2$. It is important to note that, unlike the classical definition of the $p$-norm, we do not exponentiate $\frac{1}{p}$ in \cref{eq:UQCE} to maintain consistency with prior literature. In the subsequent sections, we focus our analysis on PCE$_1$ and use the abbreviation PCE for brevity.
	

	One limitation of scalar metrics like PCE is their inability to provide detailed information regarding calibration errors at individual quantile levels, $\alpha_1, \dots, \alpha_M$. Instead, PIT reliability diagrams offer a visual assessment of probabilistic calibration across all quantile levels by plotting the empirical CDF of the PIT $Z$. These diagrams display the right side of \cref{eq:probabilistic_calibration} against its left side, with a perfectly calibrated model represented by a diagonal line (in the asymptotic case). \cref{fig:uqce_and_rel_diags_p_values} provides examples of such reliability diagrams, which have been employed in studies by \citet{Pinson2012-ba} and \citet{Kuleshov2018-tb}.

	\section{Related Work}
	\label{sec:related_work}
	
	\begin{figure*}[h]
		\centering
		\includegraphics[width=\textwidth]{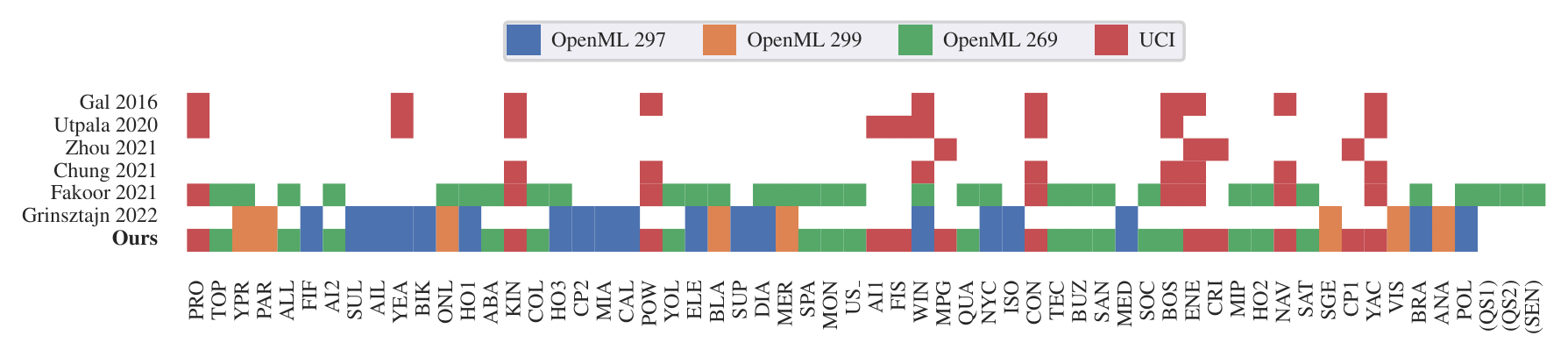}
		\caption{Multiple regression benchmark datasets with references. Datasets inside parentheses have not been considered in this study. Full dataset names are available in \cref{table:datasets}.}
		\label{fig:papers_vs_datasets}
	\end{figure*}
	
	
	
	\textbf{Post-hoc calibration approaches} involve adjusting the predictions of a trained model using a mapping learned from a separate calibration dataset. In the context of classification, temperature scaling \citep{Guo2017-ow} is a simple and effective method that adjusts predictive confidence while maintaining accuracy. For regression tasks, quantile recalibration \citep{Kuleshov2018-tb} aims to achieve probabilistic calibration. Conformal prediction \citep{Vovk2020-pg} is a general approach that provides prediction sets with a finite-sample coverage guarantee. Notable methods applied with deep learning include Conformal Quantile Regression \citep{Romano2019-kp} and Distributional Conformal Prediction \citep{Izbicki2020-ed,Chernozhukov2021-sg}. Furthermore, post-hoc approaches have also been proposed for conditional notions of calibration \citep{Song2019-bk, Kuleshov2022-pv}.

	
	\textbf{Regularization approaches} aim to improve calibration during training by incorporating regularization techniques. Some methods, proposed by \citet{Zhao2020-ze} and \citet{Feldman2021-ew}, utilize regularization to target different conditional notions of calibration based on the inputs.
	\citet{Zhou2021-hx} introduced an alternative loss function involving the simultaneous training of two neural networks, while \citet{Pearce2018-lo}, \citet{Chung2021-rh}, and \citet{Thiagarajan2020-gm} proposed objectives that allow control over the tradeoff between coverage and sharpness of prediction intervals.
	To our knowledge, the only regularization objective specifically targeting probabilistic calibration is quantile regularization \citep{Utpala2020-nw}.
	\rev{Other types of uncertainty quantification methods include ensembling \citep{Lakshminarayanan2017-zg} and Bayesian methods \citep{Jospin2022-zi}}
	
	\section{Are Neural Regression Models Probabilistically Calibrated?}
	
	\label{sec:empirical_study}
	
	
	We conduct an extensive empirical study to evaluate the probabilistic calibration of neural regression models. To this end, we calculate the \textit{probabilistic calibration error} defined in \eqref{eq:UQCE} for various state-of-the-art models across multiple benchmark datasets.
	
	
	\textbf{Benchmark datasets}. We analyze a total of 57 datasets, including 27 from the OpenML curated benchmark \citep{Grinsztajn2022-nu}, 18 from the AutoML Repository \citep{Gijsbers2019-xk}, and 12 from the UCI Machine Learning Repository \citep{Dua2017-ut}.
	
	These datasets are widely used in the evaluation of deep probabilistic models and uncertainty quantification, as evidenced by previous studies such as \citet{Fakoor2021-os, Chung2021-rh,Zhou2021-hx,Utpala2020-nw,Gal2016-xn}.
	
	\cref{fig:papers_vs_datasets} provides an overview of the utilization of these datasets in previous studies. To the best of our knowledge, our study represents the most comprehensive assessment of probabilistic calibration for neural regression models published to date.

	\textbf{Neural probabilistic regression models}. We consider three state-of-the-art neural probabilistic regression models. The first model predicts a parametric distribution, where the parameters are obtained as outputs of a hypernetwork. Previous studies have often focused on the Gaussian distribution \citep{Lakshminarayanan2017-zg, Utpala2020-nw, Zhao2020-ze}. To introduce more flexibility, we consider a mixture of $K$ Gaussian distributions. Given an input $x \in \mc{X}$, the hypernetwork parametrizes the means $\mu_k(x)$, standard deviations $\sigma_k(x)$, and weights $w_k(x)$ for each component $k = 1, ..., K$. \rev{To ensure positive standard deviations and that the mixture weights form a discrete probability distribution, we use the Softplus and Softmax activations, respectively.} We have two variants of this model depending on the scoring rule used for training: the \rev{negative log-likelihood (NLL)} or the \rev{continuous ranked probability score (CRPS)}. These models are denoted as MIX-NLL and MIX-CRPS, respectively. It is worth noting that the CRPS of a mixture of Gaussians has a closed-form expression \citep{Grimit2006-dg}.
	
	The second model predicts quantiles of the distribution \citep{Tagasovska2019-ts, Chung2021-rh, Feldman2021-ew}. Specifically, given an input $x \in \mc{X}$ and a quantile level $\alpha \in [0, 1]$, the model outputs a quantile $Q_\theta(\alpha \mid x)$. The full quantile function can be obtained by evaluating the model at multiple quantile levels. The model is trained by minimizing the quantile score at multiple levels, which is asymptotically equivalent to minimizing the CRPS \citep{Bracher2021-ut}. We denote this model as SQR-CRPS\rev{, where SQR stands for simultaneous quantile regression \citep{Tagasovska2019-ts}}.
	
    \begin{figure*}
        \centering
        \includegraphics[width=\textwidth]{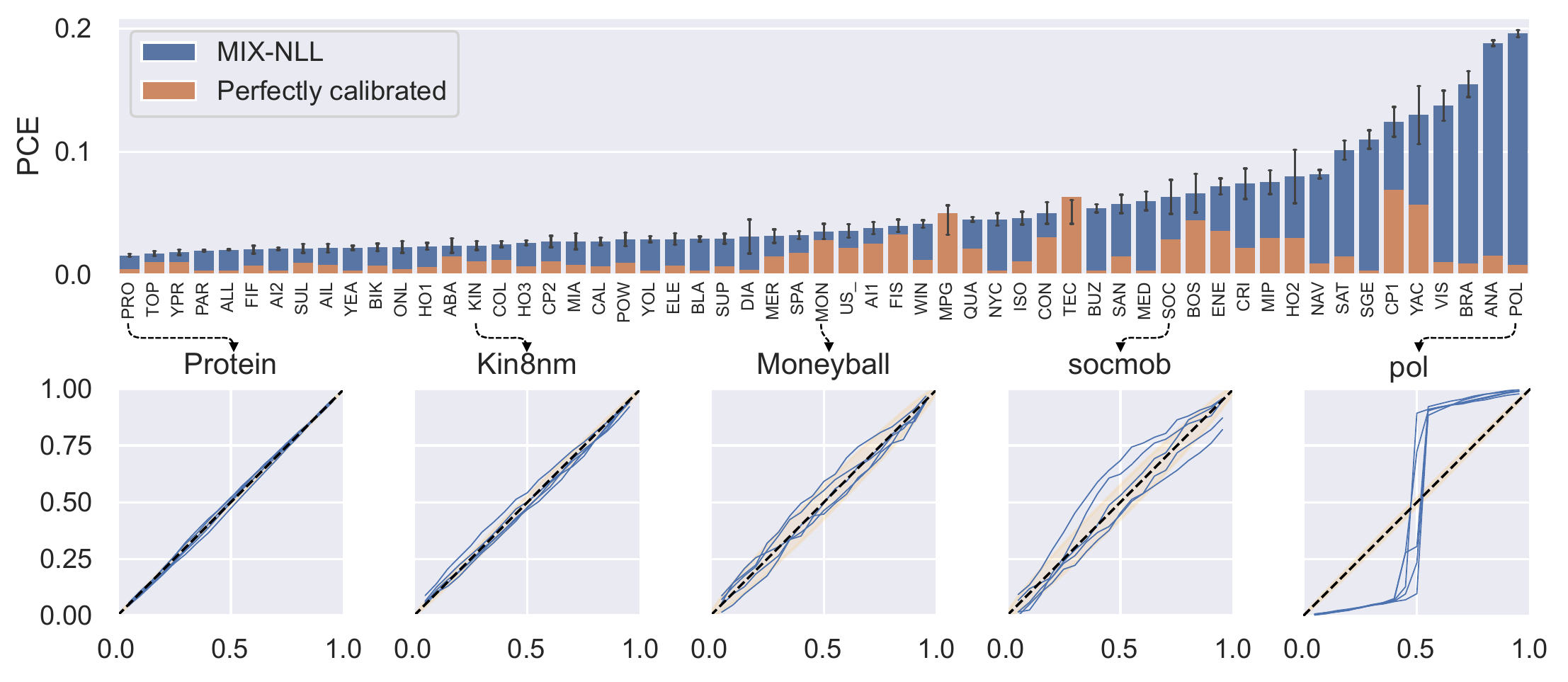}
        \caption{The top row shows the PCE for different datasets with one standard error (error bar). The bottom row gives examples of PIT reliability diagrams for five datasets.
        }
        \label{fig:uqce_and_rel_diags_p_values}
    \end{figure*}
	
	
	\textbf{Experimental setup}. We adopt the large-sized regime introduced by \citet{Grinsztajn2022-nu}, which involves truncating the datasets to a maximum of 50,000 examples. Among the 57 datasets, the number of examples ranges from 135 to 50,000, and the number of features ranges from 3 to 3,611\footnote{Please refer to Appendix \ref{sec:appendix_datasets}, specifically \cref{table:datasets}, for a detailed summary of each dataset.}. Each of the 57 datasets is divided into four sets: training (65\%), validation (10\%), calibration (15\%), and test (10\%). We normalize the input $X$ and target $Y$ using the mean and standard deviation from the training split. The final predictions are then transformed back to the original scale. For our neural network models, we employ the same fully-connected architecture as previous studies conducted by \citet{Kuleshov2018-tb}, \citet{Chung2021-rh}, and \citet{Fakoor2021-os}. Further details regarding the model hyperparameters can be found in \cref{sec:hparams}.

	\textbf{Results}. In \cref{fig:uqce_and_rel_diags_p_values}, the first row displays the PCE (averaged over five random train-validation-test splits) for MIX-NLL in blue on each of the 57 datasets. For comparison, the PCE of a perfectly calibrated model, i.e. with uniformly distributed PITs, computed using $5 \times 10^4$ simulated values is shown in orange. The second row presents reliability diagrams for five datasets, with 90\% consistency bands as in \citet{Gneiting2023-at}. Similar information is provided for MIX-CRPS and SQR-CRPS in Figures \ref{fig:pce_and_rel_diags/vanilla_wis} and \ref{fig:pce_and_rel_diags/vanilla_crps} in \cref{sec:calibration_vanilla}, respectively. Additionally, reliability diagrams for all datasets can be found in Figure \ref{fig:reliability_base_loss} in \cref{sec:rel_diags}.
	
	The analysis reveals that the (average) PCE is generally high across many datasets, although there are significant variations between datasets. To test the statistical significance of these results, $10^4$ samples were generated from the sampling distribution of the average PCE under the null hypothesis of probabilistic calibration. The resulting sampling distribution for all datasets is presented in \cref{sec:distribution_test_statistic}.

	By computing the p-value associated with a one-sided test in the upper tail of the distribution (as illustrated in \cref{sec:distribution_test_statistic}), it was observed that most datasets have a p-value of zero. This indicates that the average PCE obtained for the considered model is higher than all the simulated average PCEs of the probabilistically calibrated model. Applying a threshold of 0.01 and a Holm correction for the 57 hypothesis tests, the null hypothesis is rejected for 11 datasets out of the 57.
	
	Overall, the results indicate that the neural models considered in this study are generally probabilistically miscalibrated on a significant number of benchmark tabular datasets. In \cref{sec:comparative_study}, we will further explore how calibration methods can substantially improve the PCE of neural models.

	
	\section{Calibration Methods}
	\label{sec:methods}
	
	We begin by discussing the three main approaches to calibration: quantile recalibration, conformal prediction, and regularization-based calibration. Following that, we introduce two novel variants of regularization-based calibration.
	
	Quantile recalibration and conformal prediction are post-hoc methods, meaning they are applied after model training. These approaches utilize a separate calibration dataset $\D' = \Set{ (X'_i, Y'_i) }_{i=1}^{N'}$, where $(X'_i, Y'_i) \overset{\text{i.i.d.}}{\sim} P_{X, Y}$. On the other hand, regularization-based calibration operates directly during training and relies solely on the training data $\D$.

	
	
	\subsection{Quantile Recalibration}
	\label{sec:quantile_recalibration}
	
	

	
	Quantile recalibration aims to transform a potentially miscalibrated CDF $F_\theta$ into a probabilistically calibrated CDF $F'_\theta = F_Z \circ F_\theta$, using the calibration map $F_Z$ which represents the CDF of the PITs for $F_\theta$. For a given quantile level $\alpha \in [0, 1]$, the recalibrated CDF $F'_\theta$ satisfies:
	\begin{align}
		\text{Pr}(F'_\theta(Y \mid X) \leq \alpha)
		& = \text{Pr}(F_\theta(Y \mid X) \leq Q_Z(\alpha)) \\
		& = F_Z(Q_Z(\alpha)) \\
		& = \alpha.
	\end{align}
	In practice, $F_Z$ is not directly available and needs to be estimated from data. \citet{Kuleshov2018-tb} proposed estimating it using isotonic regression, while \citet{Utpala2020-nw} showed that computing the empirical CDF is an equivalent and simpler method. Specifically, given a set of PIT values $Z'_i = F_\theta(Y'_i \mid X'_i), i=1,\dots,N'$, the calibration map $\phi^\text{EMP}$ is computed as:
	\begin{equation}
		\label{eq:ecdf_pit}
		\phi^\text{EMP}(\alpha; \Set{Z'_i}_{i=1}^{N'}) = \frac{1}{N'} \sum_{i=1}^{N'} \indicator(Z'_i \leq \alpha),
	\end{equation}
	where $\alpha \in [0, 1]$.
	
	Similarly to \citet{Utpala2020-nw}, we also consider a linear calibration map $\phi^\text{LIN}$, which is continuous, and corresponds to a linear interpolation between the points $\Set{(0, 0), (Z'_{(1)}, \nicefrac{1}{N' + 1}), \dots, (Z'_{(N')}, \nicefrac{N'}{N' + 1}), (1, 1)}$, where $Z'_{(k)}$ is the $k$th order statistic of $Z'_1, \dots, Z'_{N'}$.
	
	In addition, we propose a calibration map based on kernel density estimation (KDE), denoted as $\phi^\text{KDE}$. This calibration map offers the advantage of being differentiable and can lead to improved NLL performance. The key idea is to use a relaxed approximation of the indicator function, which allows us to make the PIT CDF \cref{eq:ecdf_pit} differentiable. Specifically, we compute
	\[
	\indicator_\tau(a \leq b) = \sigma(\tau (b - a)) \approx \indicator(a \leq b),
	\]
	where $\tau > 0$ is an hyperparameter and $\sigma(x) = \frac{1}{1 + e^{-x}}$ denotes the sigmoid function. The resulting smoothed empirical CDF is given by
	\begin{equation}
		\label{eq:phi_kde}
		\phi^\text{KDE}(\alpha; \Set{Z'_i}_{i=1}^{N'}) = \frac{1}{N'} \sum_{i=1}^{N'} \indicator_\tau(Z'_i \leq \alpha).
	\end{equation}
	\rev{This corresponds to estimating the CDF $F_Z$ using KDE based on $N'$ realizations of $Z$ ($\Set{Z'_i}_{i=1}^{N'}$). Since $\sigma$ is the CDF of the logistic distribution, we use the PDF of the logistic distribution as the kernel in the KDE.} 	\cref{algo:quantile_recalibration} summarises this method.


	
	\begin{algorithm}[H]
		\caption{Quantile recalibration}
		\label{algo:quantile_recalibration}
		\begin{algorithmic}
			\STATE {\bfseries Input:} Predictive CDF $F_\theta$ and $\D' = \Set{(X'_i, Y'_i)}_{i=1}^{N'}$.
			\STATE Compute $Z'_i = F_\theta(Y'_i \mid X'_i) \quad (i = 1, \dots, N')$
			\STATE Compute the calibration map $\phi$, either $\phi^\text{EMP}$, $\phi^\text{LIN}$, or $\phi^\text{KDE}$
			\STATE {\bfseries Return:} Recalibrated CDF $F'_\theta = \phi \circ F_\theta$.
		\end{algorithmic}
	\end{algorithm}
	
	
	\subsection{Conformal Prediction}
	\label{sec:conformal_prediction}

	Let us assume the realizations of our calibration dataset $\mathcal{D}'$ are drawn exchangeably from $P_{X, Y}$\footnote{This is implied by the common i.i.d. assumption.}. Given a predictive model $M_{\theta}$ and a coverage level $\alpha \in [0, 1]$, (inductive) conformal prediction allows us to construct a prediction set $C_\alpha(X) \subseteq \mathcal{Y}$ for any input $X$, satisfying the property:
	\begin{align}
		\text{Pr}(Y \in C_\alpha(X))
		& = \frac{\lceil(N'+1)\alpha\rceil}{N'+1} \label{eq:conformal_equality} \\
		& \approx \alpha. \label{eq:conformal_lower_bound}
	\end{align}
	Conformal prediction achieves this by utilizing a conformity score $s_\theta(Y \mid X)$, which intuitively quantifies the similarity between new samples and previously observed samples. When the conformity score increases with $Y$, an interval $C_\alpha(X) = (-\infty, s_\theta^{-1}(\alpha \mid X)]$ can be constructed, ensuring the conformal guarantee \eqref{eq:conformal_equality} at level $\alpha$.

	
	Let $Q'_\theta(\alpha \mid X) = s_\theta^{-1}(\alpha \mid X)$ represent the (revised) model obtained through conformal prediction from $Q_\theta(\alpha \mid X)$. Under the assumption that $Q'_\theta$ is continuous and strictly increasing, the conformal guarantee implies that $\text{Pr}(Y \leq Q'_\theta(\alpha \mid X)) \approx \alpha$, which indicates approximate probabilistic calibration at level $\alpha$.
	
	Conformalized Quantile Regression \citep{Romano2019-kp} is an example of a conformal procedure, where the conformity score is defined as $s_\theta(Y \mid X) = Y - Q_\theta(\alpha \mid X)$, representing the quantile residual. Another example is Distributional Conformal Prediction \citep{Izbicki2020-ed,Chernozhukov2021-sg}, which employs the conformity score $s_\theta(Y \mid X) = F_\theta(Y \mid X)$, referring to the PIT. \cref{algo:CP} provides a summary of how to compute calibrated quantiles using inductive conformal prediction.
	
	\begin{algorithm}[h]
	\caption{Calibrated quantiles with conformal prediction}
	\label{algo:CP}
	\begin{algorithmic}
			\STATE {\bfseries Input:} Trained model $M_\theta$, $\D' = \Set{(X'_i, Y'_i)}_{i=1}^{N'}$, strictly increasing conformity score $s$, quantile level $\alpha \in [0, 1]$, input $X$.
			
			\STATE Compute $S_i = s_\theta(Y'_i \mid X'_i) \quad (i = 1, \dots, N')$\;
			\STATE Compute $\hat{q} = S_{(\ceil{(N'+1)\alpha})}$ where $S_{(k)}$ denote the $k$th smallest value among $\Set{S_1, \dots, S_{N'}, +\infty}$ \;
			\STATE {\bfseries Return:} Calibrated quantile $Q'_\theta(\alpha \mid X) = s_\theta^{-1}(\hat{q} \mid X)$
		\end{algorithmic}
	\end{algorithm}
	
	\subsection{Regularization-based Calibration}
	\label{sec:regularization}
	
	%
	
	
	
	Regularization-based calibration methods aim to enhance model calibration by incorporating a regularization term into the training objective. Although widely used in classification, there are relatively fewer methods specifically designed for regression problems. In this section, we discuss two approaches: quantile regularization \citep{Utpala2020-nw} and the truncation method \citep{Chung2021-rh}. The main steps of regularization-based calibration are summarized in Algorithm \ref{algo:regul}.

	
	
	
	\begin{algorithm}[H]
		\caption{Regularization-based calibration}
		\label{algo:regul}
		\begin{algorithmic}
			\STATE {\bfseries Input:}  Model $M_{\theta}$, calibration regularizer $\mc{R}(\theta)$ and tuning parameter $\lambda \geq 0$.
			\STATE Compute $\theta^{*} = \argmin_{\theta \in \Theta} \mc{L}^{'}(\theta; \D)$ where
			\STATE $\mc{L}^{'}(\theta; \D) = \nicefrac{1}{N} \sum_{i=1}^N \mc{L}(M_\theta(\cdot \mid X_i), Y_i) + \lambda \mc{R}(\theta; \D)$
			\STATE {\bfseries Return:} Regularized model $M_{\theta^{*}}$
		\end{algorithmic}
	\end{algorithm}
	
	\subsubsection{Quantile Regularization}
	\label{sec:entropy_based}
	
	The regularizer proposed by \citet{Utpala2020-nw} aims to measure the deviation of the PIT variable $Z$ from a uniform distribution, which is characteristic of a probabilistically calibrated model. This regularization penalty encourages the selection of calibrated models during training.
	
	The authors observed that the KL divergence between $Z$ and a uniform random variable is equivalent to the negative differential entropy of $Z$, denoted as $H(Z)$. To approximate $H(Z)$, they employed sample-spacing entropy estimation \citep{Vasicek1976-fa}, resulting in the following regularizer:
	\begin{align}
		& \mc{R}_{\text{QR}}(\theta; \mc{D})  \\
		& = \frac{1}{N-k} \sum_{i=1}^{N-k} \log \left[ \frac{N+1}{k} (Z_{(i+k)} - Z_{(i)}) \right] \\
		& \approx H(Z),
	\end{align}
	where $k$ is a hyperparameter satisfying $1 \leq k \leq N$, and $Z_{(i)}$ represents the $i$th order statistic of $Z$.
	
	To ensure differentiability during optimization, the authors employed a differentiable relaxation technique called NeuralSort \citep{Grover2019-iw}, as sorting is a non-differentiable operation.

	
	
	
	%
	
	\subsubsection{Truncation-based Calibration}
	

	The regularization approach introduced by \citet{Chung2021-rh}, that we denote Trunc, involves truncating the predictive distribution based on the current level of calibration.
	
	Given a quantile model $Q_\theta$, let $\hat{F}_Z(\alpha) = \frac{1}{N} \sum_{i=1}^N \indicator(Y_i \leq Q_\theta(\alpha \mid X_i))$ be the estimated PIT CDF evaluated at $\alpha$ and $\rho(x, y) = (y - x) \indicator(x < y)$. The regularization objective for level $\alpha$ is defined as follows:
	\begin{align}
		& \mc{R}_\text{Trunc}(\theta; \mc{D}, \alpha) \\
		& =
		\begin{cases}
			\frac{1}{N} \sum_{i=1}^N \rho(Q_\theta(\alpha \mid X_i), Y_i) & \text{if $\hat{F}_Z(\alpha) < \alpha$} \\
			\frac{1}{N} \sum_{i=1}^N \rho(Y_i, Q_\theta(\alpha \mid X_i)) & \text{otherwise}
		\end{cases}
	\end{align}
	
	This regularization objective adjusts $\hat{F}_Z(\alpha)$ to match $\alpha$ by increasing it when $\hat{F}_Z(\alpha) < \alpha$, and vice versa. The final regularization objective is computed by averaging $\mc{R}_\text{Trunc}(\theta; \mc{D}, \alpha)$ over multiple quantile levels $\Set{\alpha_j}_{j=1}^M$:
	\begin{equation}
		\mc{R}_\text{Trunc}(\theta; \mc{D}) = \frac{1}{M} \sum_{j=1}^M \mc{R}_\text{trunc}(\theta; \mc{D}, \alpha_j).
	\end{equation}
	

	It is worth noting that \citet{Chung2021-rh} combine the previous regularization objective with a sharpness objective that penalizes the width between the quantile predictions, given by $\frac{1}{M} \sum_{j=1}^M \frac{1}{N} \sum_{i=1}^N |Q_\theta(\alpha_j \mid X_i) - Q_\theta(1 - \alpha_j \mid X_i)|$. Instead, we combine it with a strictly proper scoring rule.

	
	
	
	
	

	\subsection{New Regularization-based Calibration Methods}
	\label{sec:new_regularization}
	
	Building upon the quantile calibration method discussed in Section \ref{sec:entropy_based}, we propose two new regularization objectives which compute a differentiable $\text{PCE}_p$ using alternative statistical distances.
	
	The first approach, named PCE-KDE, leverages the differentiable calibration map $\phi^\text{KDE}$ \cref{eq:phi_kde} based on KDE. Given a set of quantile levels $\Set{\alpha_j}_{j=1}^M$, the regularization objective is given by
	\begin{align}
		\mc{R}_\text{PCE-KDE}(\theta; \mc{D}) = \frac{1}{M} \sum_{j=1}^M \left|\alpha_j - \phi^\text{KDE}(\alpha_j; \Set{Z_i}_{i=1}^{N}) \right|^p, 
	\end{align}
	where $p > 0$. Note that $\mc{R}_\text{PCE-KDE}$ reduces to $\text{PCE}_p$ in \eqref{eq:UQCE} when $\tau$ in \eqref{eq:phi_kde} goes to $\infty$.

	


	The second approach considers distances of the form $\int_0^1 |Q_Z(\alpha) - Q_U(\alpha)|^p d\alpha$, where $Q_Z$ and $Q_U$ denote the quantile functions of the true and uniform distributions, respectively. When $p = 1$, this distance reduces to the 1-Wasserstein distance, equivalent to $\int_0^1 |F_Z(\alpha) - F_U(\alpha)| d\alpha$, which aligns with PCE (see \cref{th:integral_abs_diff_cdf_quantile} in  \cref{sec:integral_abs_diff_cdf_quantile}).

	
	By exploiting the fact that $\E{F_Z(Z_{(i)})} = \nicefrac{i}{N+1}$, we approximate $Q_Z(\nicefrac{i}{N + 1})$ using the $i$-th order statistic $Z_{(i)}$. The regularization objective is given by
	\begin{equation}
		\mc{R}_\text{PCE-Sort}(\theta; \mc{D}) = \frac{1}{N} \sum_{i=1}^{N} \left\lvert Z_{(i)} - \frac{i}{N + 1} \right\rvert^p,
	\end{equation}
	where $p > 0$. Differentiable relaxations to sorting, such as those proposed by \citet{Blondel2020-sf} and \citet{Cuturi2019-dv}, can be employed to obtain the order statistics.

	
	\section{A Comparative Study of Probabilistic Calibration Methods}
	\label{sec:comparative_study}
	
	In continuation of the empirical study described in  \cref{sec:empirical_study}, we now proceed to evaluate the performance of the probabilistic calibration methods outlined in the previous section. Specifically, we apply eight distinct calibration methods to the three neural regression models introduced in \cref{sec:empirical_study}. These methods are evenly divided into two categories: post-hoc methods and regularization-based methods.
	
	To assess the effectiveness of these calibration methods, we employ four different evaluation metrics. The evaluation is conducted on a set of 57 datasets, utilizing the same experimental setup detailed in  \cref{sec:empirical_study}. To ensure a fair and consistent comparison, all the methods have been implemented within a unified codebase\footnote{The code can be accessed at the following GitHub repository: \url{https://github.com/Vekteur/probabilistic-calibration-study}}.


	\subsection{Experimental Setup}
	\label{sec:experiments}

	%

	
	\paragraph{Base probabilistic models and calibration methods.}
	
	We consider the three probabilistic models presented in \cref{sec:empirical_study}, namely \texttt{MIX-NLL}, \texttt{MIX-CRPS}, and \texttt{SQR-CRPS}. For the \texttt{MIX} models, when applying quantile recalibration, we transform the CDF using the empirical CDF estimator (\texttt{Rec-EMP}), the linear estimator (\texttt{Rec-LIN}), or the KDE estimator (\texttt{Rec-KDE}). For \texttt{SQR-CRPS}, we transform multiple quantiles using conformalized quantile regression (\texttt{CQR}). For the three models, we consider the four regularization objectives presented in \cref{sec:regularization,sec:new_regularization} (with $p = 1$), namely $\mc{R}_\text{PCE-KDE}$ (\texttt{PCE-KDE}), $\mc{R}_\text{PCE-Sort}$ (\texttt{PCE-Sort}), $\mc{R}_\text{QR}$ (\texttt{QR}), and $\mc{R}_\text{Trunc}$ (\texttt{Trunc}). \texttt{PCE-Sort} is only shown in the Appendix because it performs similarly to \texttt{PCE-KDE}.
	

	
	\paragraph{Metrics.}
	
	We measure the accuracy of the probabilistic predictions using NLL and CRPS. For the SQR model, \rev{we estimate CRPS by averaging the quantile score at 64 equidistant quantile levels.}
	Probabilistic calibration is measured using PCE, defined in \eqref{eq:probabilistic_calibration}. Finally, we measure sharpness using the mean standard deviation of the predictive distributions, denoted by STD.
	
	\paragraph{Hyperparameters.}
	
	In our experiments, \texttt{MIX-NLL} and \texttt{MIX-CRPS} output a mixture of 3 Gaussians, and \texttt{SQR-CRPS} outputs 64 quantiles. We justify the choice of these hyperparameters in \cref{sec:hparams}. The hyperparameter $\tau$ of \texttt{Rec-KDE} and \texttt{PCE-KDE} is fixed at 100, which was found to perform well empirically. For regularization methods, an important hyperparameter is the regularization factor $\lambda$. As previously observed in classification \citep{Karandikar2021-ml}, we found that higher values of $\lambda$ tend to improve calibration but worsen NLL, CRPS, and STD. \citet{Karandikar2021-ml} proposed to limit the loss in accuracy by a maximum of 1\%. We adopt a similar strategy by selecting $\lambda$ which minimizes PCE with a maximum increase in CRPS of 10\% in the validation set. For each dataset, we select $\lambda$ in the set $\{0, 0.01, 0.05, 0.2, 1, 5\}$, which corresponds to various degrees of calibration regularization.
	
	\paragraph{Comparison of multiple models over many datasets.}
	
	As in \citet{Karandikar2021-ml}, and since NLL, CRPS and STD have different scales across datasets, we report Cohen's d, which is a standardized effect size comparing the mean of one method (over 5 runs, in our case) against a baseline. Values of $-0.8$ and $-2$ are considered large and huge, respectively. Due to the heterogeneity of the datasets that we consider, the performance of our models can vary widely across datasets. To visualize the results, we show the distribution of Cohen's d using letter-value plots \citep{Hofmann2011-rg}, which indicate the quantiles at levels $\nicefrac{1}{8}$, $\nicefrac{1}{4}$, $\nicefrac{1}{2}$, $\nicefrac{3}{4}$ and $\nicefrac{7}{8}$, as well as outliers. A median value below zero indicates that the model improved the metric on more than half the datasets.
	
	In order to assess whether significant differences exist between different methods, we follow the recommendations of \citet{Ismail_Fawaz2019-od}, which are based on \cite{Demsar2006-ed}. First, we test for a significant difference among model performances using the Friedman test \citep{Friedman1940-vi}. Then, we use the pairwise post-hoc analysis recommended by \citet{Benavoli2016-mh} using a Wilcoxon signed-rank test \citep{Wilcoxon1945-cp} with Holm's alpha correction \citep{Holm1979-rk}. The results of this procedure are shown using a critical difference diagram \citep{Demsar2006-ed}. The lower the rank (further to the right), the better performance of a model. A thick horizontal line shows a group of models whose performance is not significantly different, with a significance level of 0.05.
	
	\subsection{Results}
	
	Figure \ref{fig:posthoc_and_regul_vs_baseline/test_calib_l1} shows the letter-values plots for the Cohen’s d of PCE (top panel) as well as the associated critical diagram (bottom panel), for all methods and datasets. The reference model is \texttt{MIX-NLL}. The results with other models as reference are available in \cref{sec:comparison_per_base_model}.  Blue, green, and red colors are used for the post-hoc methods, the regularization-based methods, and the base models, respectively. The same information is given in Figures \ref{fig:posthoc_and_regul_vs_baseline/test_wis} and \ref{fig:posthoc_and_regul_vs_baseline/test_nll} for the CRPS and the NLL, respectively.
	
	\paragraph{Comparison of PCE.} As expected, \cref{fig:posthoc_and_regul_vs_baseline/test_calib_l1} shows that the PCE of calibration methods is improved compared to the base models. Furthermore, independently of the base model, we can see that post-hoc methods achieve significantly better PCE than regularization methods. When comparing \texttt{PCE-KDE} with \texttt{QR}, we can see that there is a significantly larger decrease in PCE with the \texttt{MIX-CRPS} base model compared to \texttt{MIX-NLL}. Finally, both \texttt{PCE-KDE} and \texttt{Trunc} decrease PCE for \texttt{SQR-CRPS}, without a significant difference between them.

	\begin{figure}
		\centering
		\subcaptionbox{
			Cohen's d of PCE with respect to the \texttt{MIX-NLL} model
			\label{fig:posthoc_and_regul_vs_baseline/test_calib_l1/main_metrics/cohen_d_boxplot}
		}{
			\includegraphics[width=\linewidth]{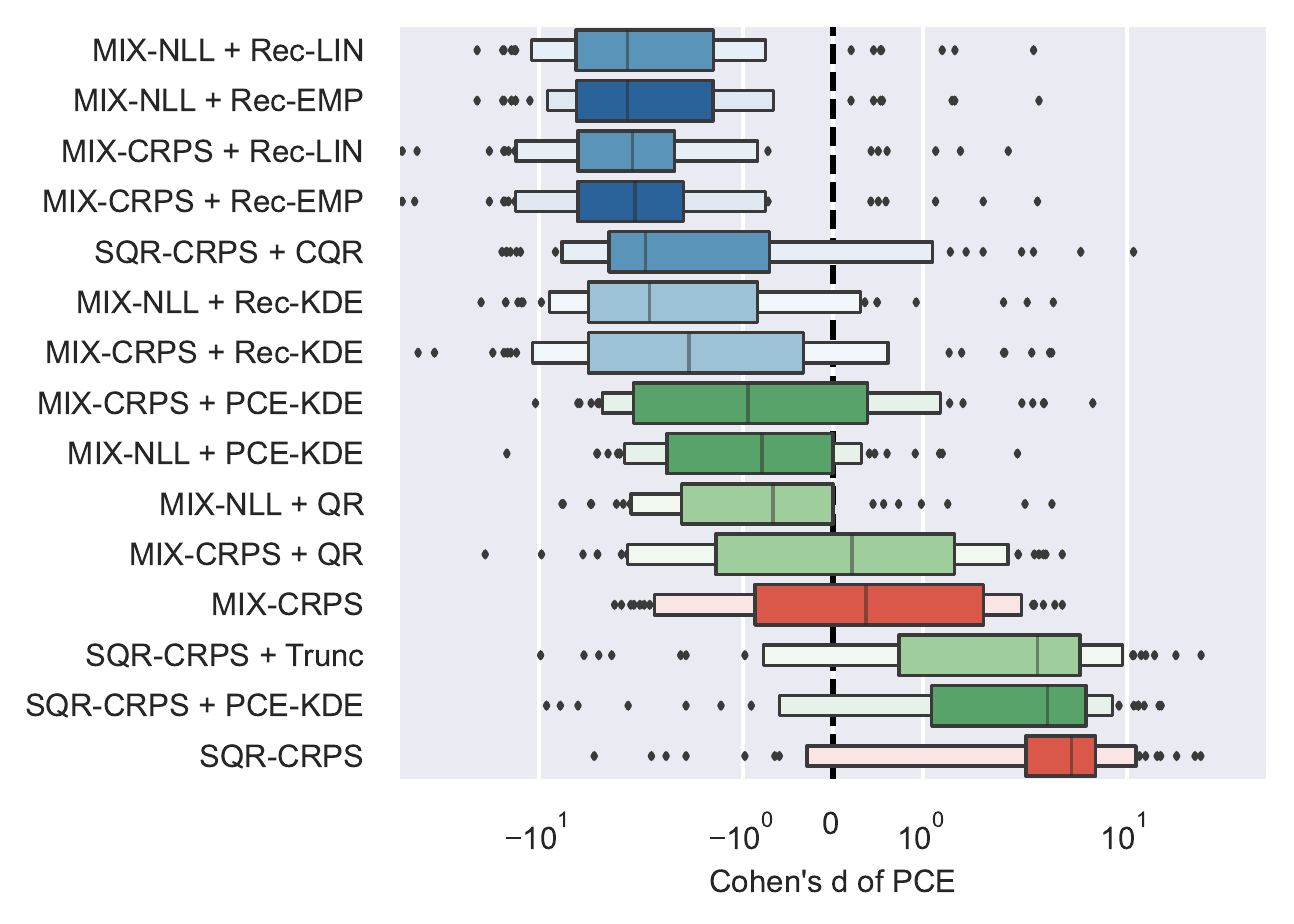}
			\vspace{-0.3cm}
		}
		\subcaptionbox{
			Critical difference diagram
			\label{fig:posthoc_and_regul_vs_baseline/cd_diagrams/test_calib_l1}
		}{
			\includegraphics[width=\linewidth]{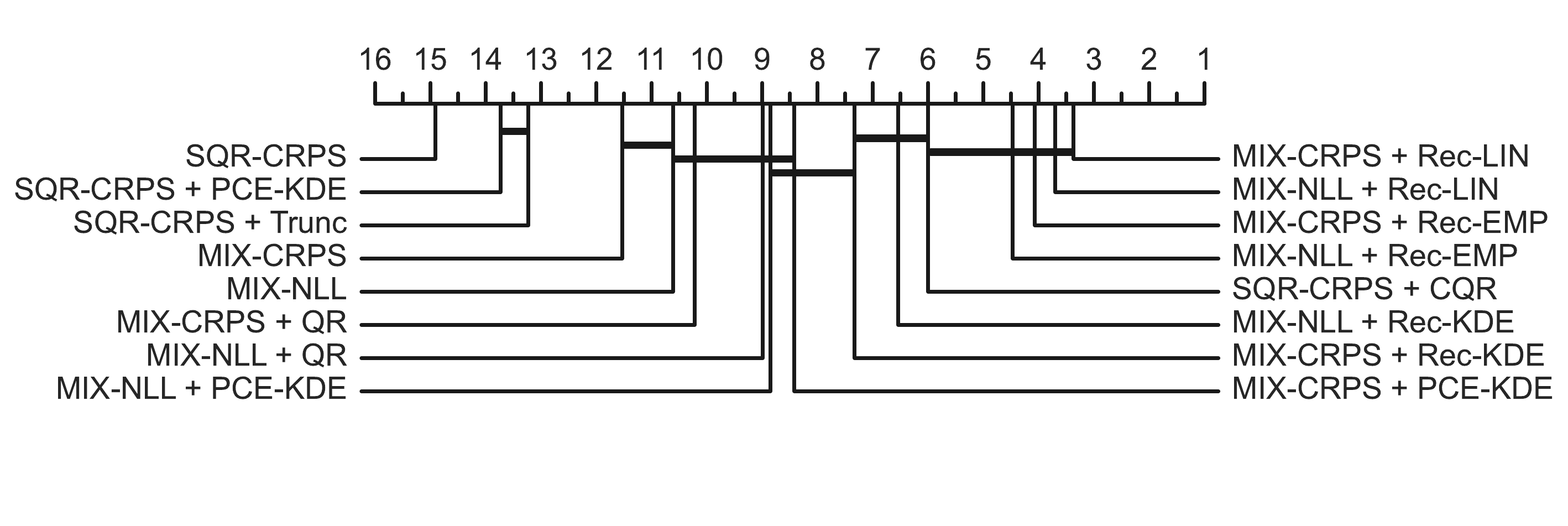}
			\vspace{-0.6cm}
		}
		\vspace{-0.2cm}
		\caption{Comparison of PCE with multiple base losses and calibration methods.}
		\label{fig:posthoc_and_regul_vs_baseline/test_calib_l1}
	\end{figure}
	
	\paragraph{Comparison of CRPS.} While post-hoc methods outperform regularization methods in terms of PCE, \cref{fig:posthoc_and_regul_vs_baseline/test_wis} shows they have a higher CRPS (except for the \texttt{SQR} base model). This can be explained by the fact that regularization methods prevent the CRPS from increasing exceedingly due to the selection criterion for $\lambda$.
	
	\paragraph{Comparison of NLL.}
	
	\cref{fig:posthoc_and_regul_vs_baseline/test_nll} shows the importance of the calibration map. In fact, quantile recalibration with a linear map significantly increases the NLL, while smooth interpolation decreases PCE without a large increase in NLL. Note that we only consider \texttt{MIX} models since we cannot compute the NLL for \texttt{SQR}.

	\begin{figure}
		\centering
		\subcaptionbox{
			Cohen's d of CRPS with respect to the \texttt{MIX-NLL} model
			\label{fig:posthoc_and_regul_vs_baseline/test_wis/main_metrics/cohen_d_boxplot}
		}{
			\includegraphics[width=\linewidth]{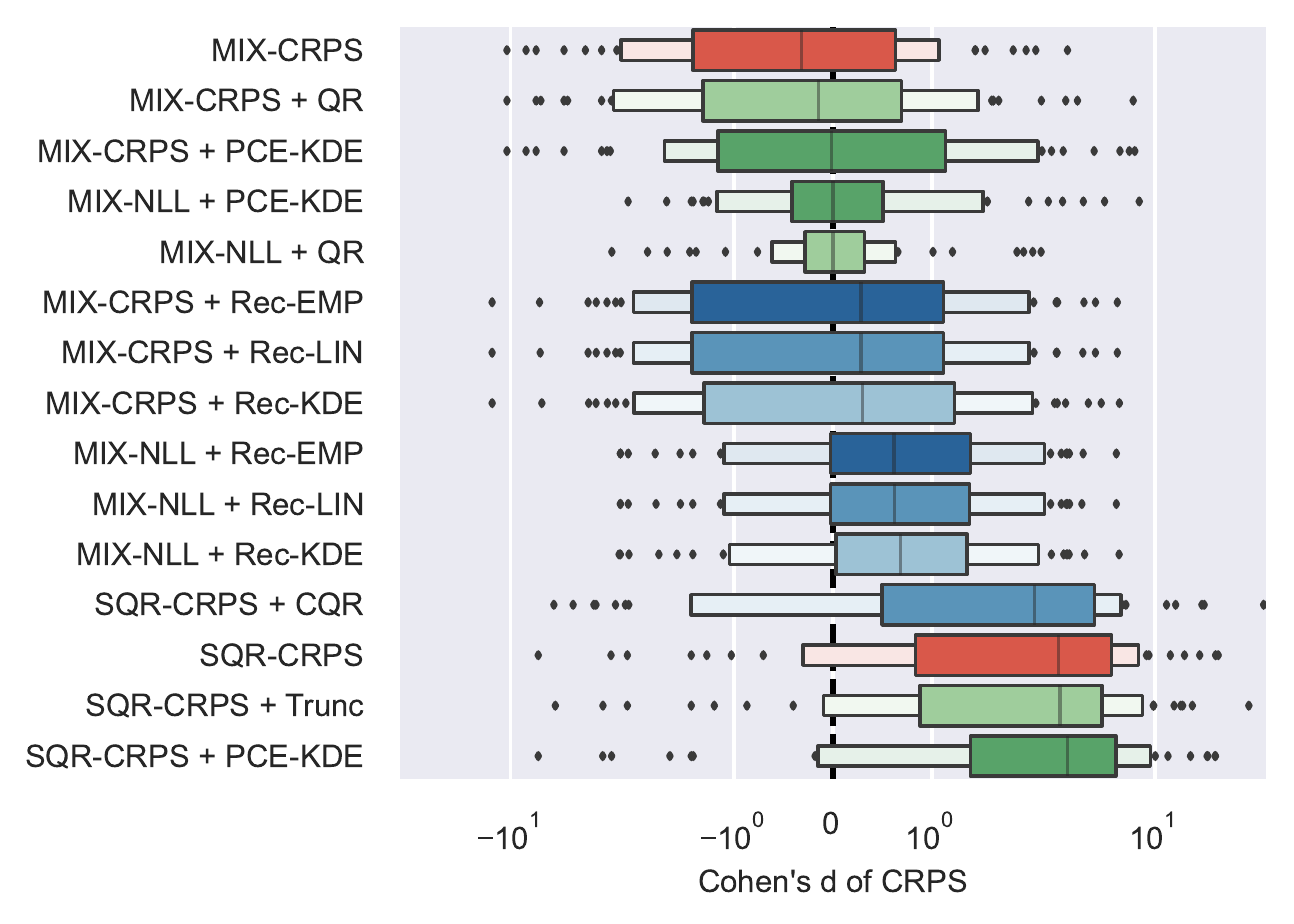}
			\vspace{-0.3cm}
		}
		\subcaptionbox{
			Critical difference diagram
			\label{fig:posthoc_and_regul_vs_baseline/cd_diagrams/test_wis}
		}{
			\includegraphics[width=\linewidth]{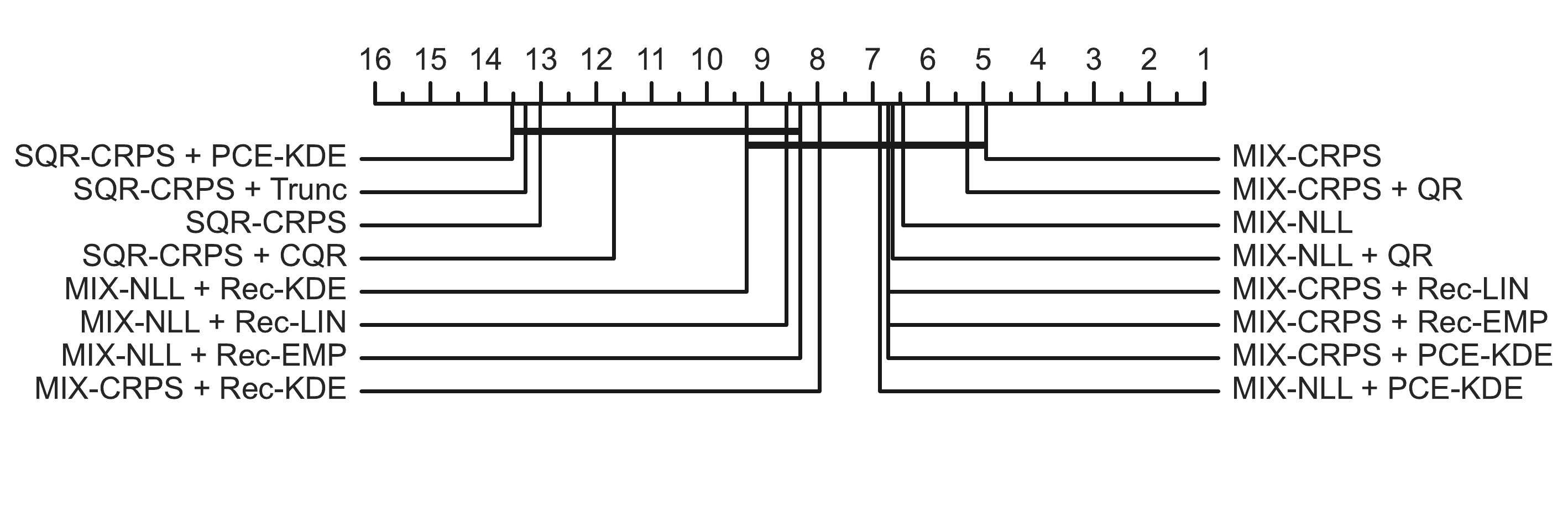}
			\vspace{-0.6cm}
		}
		\vspace{-0.2cm}
		\caption{Comparison of CRPS with multiple base losses and calibration methods.}
		\label{fig:posthoc_and_regul_vs_baseline/test_wis}
	\end{figure}

	\paragraph{On the choice of a calibration method.}
	
	If probabilistic calibration is critical to the application, our experiments suggest that post-hoc methods such as quantile recalibration and conformal prediction should be preferred. However, when we also want to control the CRPS or the NLL, regularization methods can offer a better trade-off in terms of calibration and sharpness. In fact, as shown in \cref{fig:posthoc_or_regul_vs_vanilla/nll} in \cref{sec:comparison_per_base_model}, when the base model is \texttt{MIX-NLL}, all regularization methods provide a significant improvement in probabilistic calibration without deteriorating the CRPS, NLL or STD. For the \texttt{MIX-CRPS} model, \cref{fig:posthoc_or_regul_vs_vanilla/crps} shows that \texttt{QR} has limited impact on CRPS and NLL, while providing better calibration.
	For the \texttt{SQR-CRPS} base model, \cref{fig:posthoc_or_regul_vs_vanilla/wis} shows that the \texttt{SQR-CRPS + CQR} conformal method significantly outperforms the \texttt{Trunc} and \texttt{PCE-KDE} regularization methods both in terms of PCE and CRPS.
	\rev{Overall, \cref{sec:comparison_per_base_model} suggests that \texttt{MIX-NLL + PCE-KDE}, \texttt{MIX-CRPS + QR} and \texttt{SQR-CRPS + CQR} are good choices for practitioners aiming to improve PCE without significantly impacting other aspects of the conditional distribution.}
	Finally, since both regularization and post-hoc methods are able to improve calibration, we investigate whether a combination of these two methods can lead to better performance. \cref{fig:posthoc_and_regul_vs_posthoc} in \cref{sec:combining_regul_and_post_hoc} shows that such a combination does not significantly improve probabilistic calibration, with an increase in CRPS and NLL.
	\rev{This indicates that practitioners should exercise caution when applying regularization to a model that is already well-calibrated.}
	
	\begin{figure}
		\centering
		\subcaptionbox{
			Cohen's d of NLL with respect to the \texttt{MIX-NLL} model
			\label{fig:posthoc_and_regul_vs_baseline/test_nll/main_metrics/cohen_d_boxplot}
		}{
			\includegraphics[width=\linewidth]{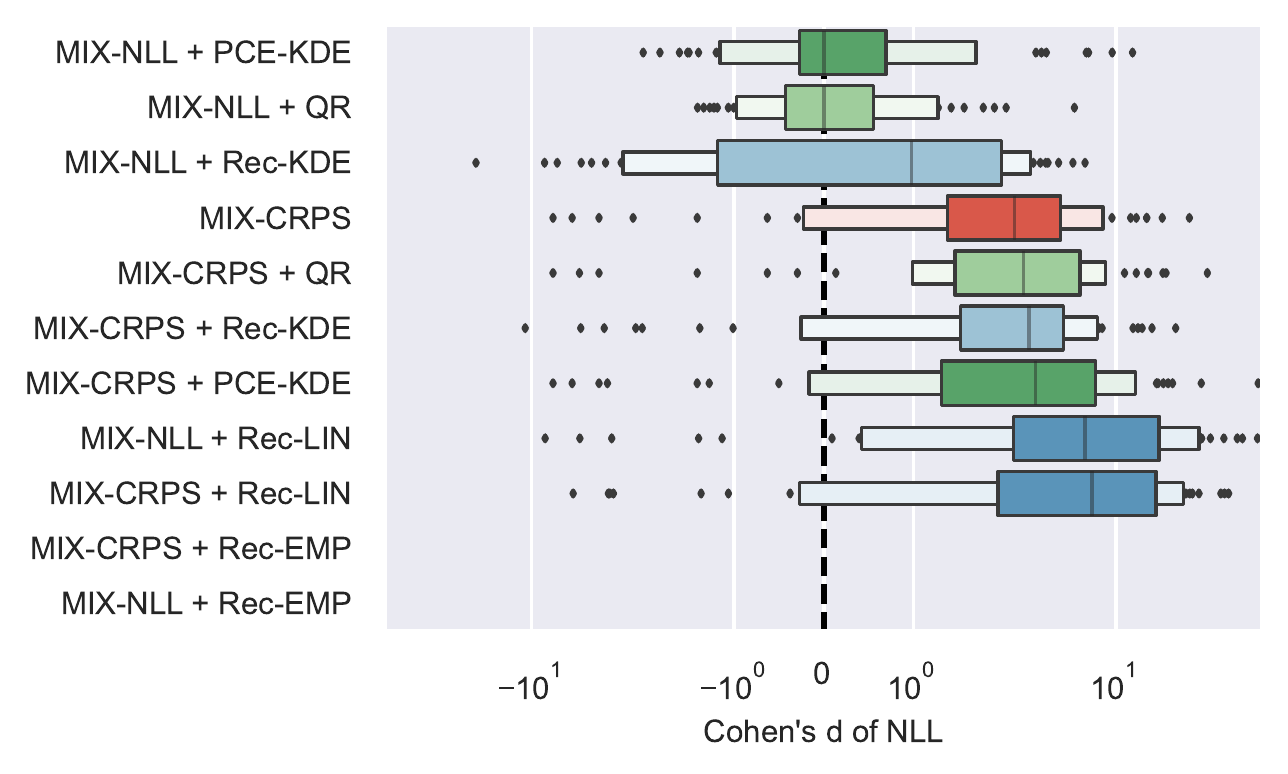}
			\vspace{-0.3cm}
		}
		\subcaptionbox{
			Critical difference diagram
			\label{fig:posthoc_and_regul_vs_baseline/cd_diagrams/test_nll}
		}{
			\includegraphics[width=\linewidth]{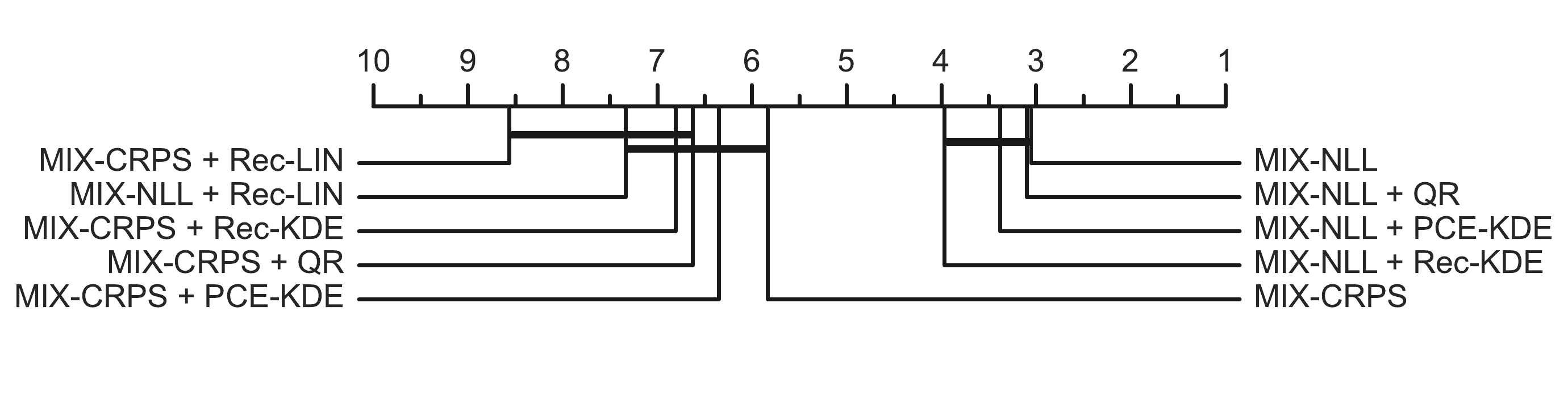}
			\vspace{-0.6cm}
		}
		\vspace{-0.2cm}
		\caption{Comparison of NLL with multiple base losses and calibration methods.}
		\label{fig:posthoc_and_regul_vs_baseline/test_nll}
	\end{figure}
	
	\subsection{Link between Quantile Recalibration and Conformal Prediction}
	
	Conformal prediction methods are well-known for their finite-sample coverage guarantee. Interestingly, a specific implementation of quantile recalibration can be considered a special case of conformal prediction. This implies that quantile recalibration can also provide a finite-sample coverage guarantee. This observation could potentially explain why both methods, conformal prediction and quantile recalibration, are effective in improving probabilistic calibration.
	
	
	\begin{theorem}
		\label{th:eq_quantile_recalibration_and_dcp}
		Quantile recalibration is equivalent to Distributional Conformal Prediction (DCP) of left intervals at each coverage level $\alpha \in [0, 1]$. The equivalence is obtained when the estimator of the calibration map is defined by a slightly different estimator than the conventional one in \eqref{eq:ecdf_pit}, namely $\rev{\phi_{\text{DCP}}}(\alpha) = \frac{1}{N'+1} \sum_{i=1}^{N'} \indicator(Z'_i \leq \alpha)$.
	\end{theorem}
	
	\begin{proof}
	Given a predictive distribution $F_\theta$ learned from a training dataset $\D = \Set{(X_i, Y_i)}_{i=1}^N$ where $(X_i, Y_i) \overset{\text{i.i.d.}}{\sim}P_{X, Y}$, let $Z'_i = F_\theta(Y'_i \mid X'_i)$ represent the PIT values computed on a separate calibration dataset $\D' = \Set{ (X'_i, Y'_i) }_{i=1}^{N'}$, where $(X'_i, Y'_i) \overset{\text{i.i.d.}}{\sim} P_{X, Y}$.
	
	In the DCP approach, as outlined in \cref{algo:CP}, the conformal scores are given by the PIT values $Z'_i$. DCP first computes the $\alpha$ empirical quantile of the scores as $\hat{q} = Z'_{(\ceil{(N'+1)\alpha})}$, where $Z'_{(k)}$ represents the $k$th smallest value among $\Set{Z'_1, \dots, Z'_{N'}, +\infty}$. Then, the conformalized quantile is computed as $Q'_\theta(\alpha \mid X) = Q_\theta(\hat{q} \mid X)$, which corresponds to conformal prediction with coverage $\alpha$ for the left interval $(-\infty, Q'_\theta(\alpha \mid X)]$.
	
	Let us consider quantile recalibration with the calibration map $\phi$ in \cref{algo:quantile_recalibration} given by $\phi_{\text{DCP}}(\alpha) = \frac{1}{N'+1} \sum_{i=1}^{N'} \indicator(Z'_i \leq \alpha)$. It computes a recalibrated CDF $F'_\theta$ by composing the original CDF $F_\theta$ with $\phi_{\text{DCP}}$, yielding $F'_\theta(y \mid X) = \phi_\text{DCP}(F_\theta(y \mid X))$.
	
	We observe that $\phi_\text{DCP}$ is the CDF of a discrete random variable, with $\phi_\text{DCP}^{-1}(\alpha) = Z'_{(\ceil{(N' + 1) \alpha})}$ representing its empirical quantile function. Furthermore, the composition $\phi_\text{DCP} \circ F_\theta(\cdot \mid X)$ acts as the inverse function of $Q_\theta(\cdot \mid X) \circ \phi_\text{DCP}^{-1}$. As a result, both the DCP approach and quantile recalibration yield QFs and CDFs that correspond to the same underlying distribution.

	
	
	

	\end{proof}
	
	Quantile recalibration with other recalibration maps (e.g., $\phi^\text{EMP}$, $\phi^\text{LIN}$, or $\phi^\text{KDE}$) would correspond to DCP where the empirical quantile $\hat{q}$ is selected using other strategies which does not provide the exact conformal guarantee \cref{eq:conformal_equality}.

	\section{Conclusion}

	The observation that neural network classifiers tend to be miscalibrated \citep{Guo2017-ow} has prompted the development of various approaches for calibrating these models. In this paper, we present the largest empirical study conducted to date on the probabilistic calibration of neural regression models. Our study provides valuable insights into their performance and the selection of calibration methods. Notably, we introduce a novel differentiable calibration map based on kernel density estimation for quantile recalibration, as well as two novel regularization objectives derived from the PCE.

	Our study reveals that regularization methods can provide a favorable tradeoff between calibration and sharpness. However, post-hoc methods demonstrate superior performance in terms of PCE. We attribute this finding to the finite-sample coverage guarantee offered by conformal prediction and demonstrate that quantile recalibration can be viewed as a specific case of conformal prediction.
	
	Future investigations may extend the study of probabilistic calibration to other models, such as tree-based models, and explore alternative notions of calibration \citep{Gneiting2021-vu}. Notably, distribution calibration represents a promising direction, as it has inspired the development of calibration methods \citep{Song2019-bk,Kuleshov2022-pv}.
	
    \section*{Acknowledgement}

    This work was supported by the Fonds de la Recherche Scientifique - FNRS under Grants T.0011.21 and J.0011.20.

	\nocite{Gal2016-xn}
	

	\printbibliography
	
	\newpage
	\appendix
	\onecolumn

	\section{Proofs}
	
	%
	%
	
	\subsection{Integral of the Absolute Difference between CDFs or QFs}
	\label{sec:integral_abs_diff_cdf_quantile}
	
	\begin{proposition}
		\label{th:integral_abs_diff_cdf_quantile}
		Let $F_A, F_B : [0, 1] \to [0, 1]$ denote two strictly increasing CDFs of random variables defined on $[0, 1]$ with corresponding QFs $Q_A$ and $Q_B$.
		Then,
		\begin{equation}
			\label{eq:theorem_integral_abs_diff_cdf_quantile}
			\int_0^1 |F_A(q) - F_B(q)| \,dq = \int_0^1 |Q_A(p) - Q_B(p)| \,dp.
		\end{equation}
	\end{proposition}
	
	\begin{proof}
		We define two functions $r, s : [0, 1] \times [0, 1] \to \Set{0, 1}$ where
		\begin{align}
			r(q, p) & = \begin{cases}
				1 & \text{if } F_A(q) \leq p \leq F_B(q) \text{ or } F_B(q) \leq p \leq F_A(q) \\
				0 & \text{otherwise},
			\end{cases} \\
			\text{and } s(q, p) & = \begin{cases}
				1 & \text{if } Q_A(p) \leq q \leq Q_B(p) \text{ or } Q_B(p) \leq q \leq Q_A(p) \\
				0 & \text{otherwise}.
			\end{cases}
		\end{align}
		
		Let us show that $r$ and $s$ are equal. Considering $q \in [0, 1]$ and $p \in [0, 1]$, we can write
		\begin{align}
			& F_A(q) \leq p \leq F_B(q)                                          \\
			\iff & (F_A(q) \leq p) \land (p \leq F_B(q) )                                 \\
			\iff & (q \leq Q_A(p)) \land (Q_B(p) \leq q) \label{eq:increasing_assumption} \\
			\iff & Q_B(p) \leq q \leq Q_A(p),
		\end{align}
		where \cref{eq:increasing_assumption} holds since both $F_A$ and $F_B$ are strictly increasing.
		
		Similarly, $F_B(q) \leq p \leq F_A(q) \iff Q_A(p) \leq q \leq Q_B(p)$.
		Hence $r(q, p) = 1 \iff s(q, p) = 1$ and $r$ and $s$ are equal.
		
		By Fubini's theorem, we have
		\begin{equation}
			\label{eq:double_integral_fubini}
			\int_0^1 \int_0^1 r(q, p) \,dp\,dq = \int_0^1 \int_0^1 s(q, p) \,dq\,dp.
		\end{equation}
		
		Furthermore, upon evaluating the inner integrals, we obtain
		\begin{align}
			\int_0^1 r(q, p) \,dp
			& = \begin{cases}
				\displaystyle \int_{F_A(q)}^{F_B(q)} 1 \,dp & \text{if } F_A(q) \leq F_B(q) \\
				\displaystyle \int_{F_B(q)}^{F_A(q)} 1 \,dp & \text{otherwise}              \\
			\end{cases} \\
			& = |F_A(q) - F_B(q)|.
		\end{align}
		Similarly, we have $\displaystyle \int_0^1 s(q, p) \,dq = |Q_A(p) - Q_B(p)|$. Finally, by substituting these results in \cref{eq:double_integral_fubini}, we prove \cref{eq:theorem_integral_abs_diff_cdf_quantile}.
		
	\end{proof}

	\section{Detailed Results}
	\label{sec:detailed_results}
	
	This section presents additional experimental results.

	
	
	
	

	\subsection{Comparison between Recalibration, Conformal Prediction and Regularization Approaches per Base Model}
	\label{sec:comparison_per_base_model}
	
	First, we present the results of our experiments comparing recalibration, conformal prediction, and regularization approaches. Our objective is to determine which metrics are improved by these methods compared to a vanilla model. We divide our comparisons based on the three base models considered: \texttt{MIX-NLL} (\cref{fig:posthoc_or_regul_vs_vanilla/nll}), \texttt{MIX-CRPS} (\cref{fig:posthoc_or_regul_vs_vanilla/crps}) and \texttt{SQR-CRPS} (\cref{fig:posthoc_or_regul_vs_vanilla/wis}).
	
	Since NLL, CRPS, and standard deviation cannot be directly compared across different datasets, we utilize Cohen's d as an effect size measure, with the baseline being a vanilla model of the same base model. For instance, the baseline for \texttt{MIX-CRPS + Rec-EMP} is \texttt{MIX-CRPS}. Additionally, we provide critical difference diagrams to assess the significance of differences.
	
	Overall, recalibration and conformal prediction demonstrate significantly improved PCE compared to the baseline, although there is a trade-off with other metrics. For both \texttt{MIX-NLL} and \texttt{MIX-CRPS}, \texttt{Rec-EMP} yields infinite NLL, \texttt{Rec-LIN} substantially increases NLL, while \texttt{Rec-KDE} has a lesser impact on NLL. However, \texttt{Rec-KDE} results in a significant degradation of CRPS compared to other recalibration methods when the base model is MIX-CRPS. In the case of quantile predictions, \texttt{CQR} significantly improves PCE.
	
	While regularization methods generally lead to improved PCE, they are still outperformed by recalibration and conformal prediction in this regard. However, we observe that with the \texttt{MIX-NLL} base model, regularization methods (\texttt{PCE-KDE}, \texttt{PCE-Sort} and \texttt{QR}) have minimal impact on CRPS, NLL, and STD compared to recalibration methods. With the \texttt{MIX-CRPS} base model, the difference in CRPS between recalibration and regularization is less pronounced. Nevertheless, it is evident that regularization methods \texttt{PCE-KDE} and \texttt{PCE-Sort}, which rely on PCE, result in less sharp predictions compared to recalibration methods, which produce sharper predictions.
	
	Regarding quantile predictions, the case is reversed: conformal prediction (\texttt{SQR-CRPS + CQR}) yields less sharp predictions, while regularization with \texttt{SQR-CRPS + Trunc} leads to sharper predictions.

	
	
	
	\begin{figure}[H]
		\centering
		\subcaptionbox{
			Boxplots of Cohen's d of different metrics on all datasets, with respect to \texttt{MIX-NLL}.
		}{
			\includegraphics[width=\linewidth]{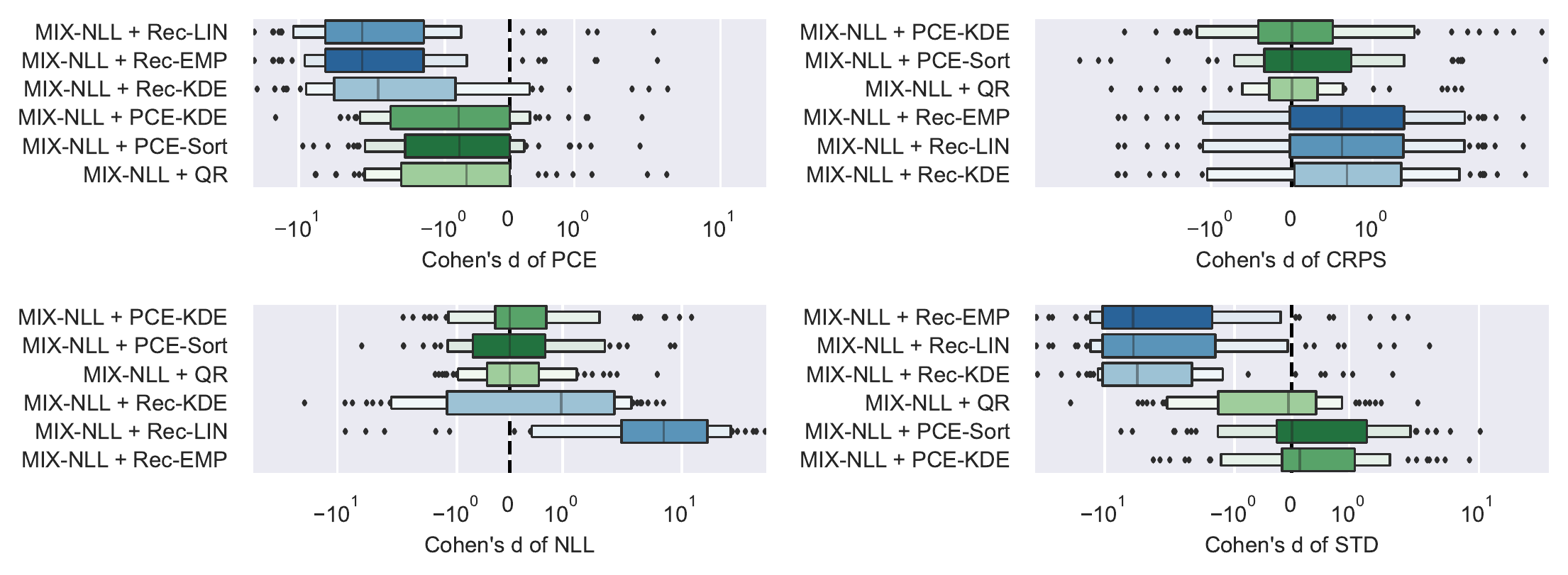}
			\vspace{-0.3cm}
		}
		\subcaptionbox{
			PCE
		}{
			\includegraphics[width=8cm]{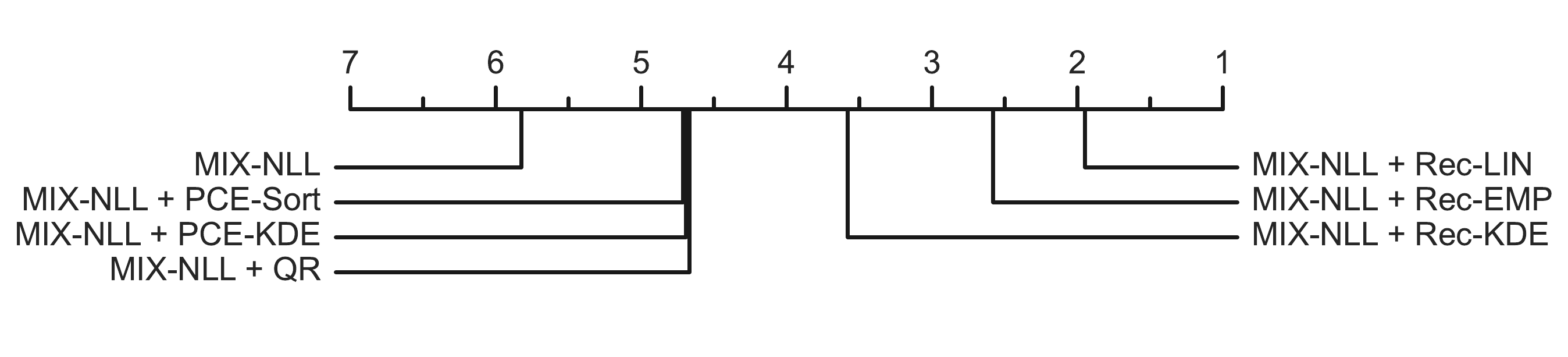}
			\vspace{-0.3cm}
		}
		\subcaptionbox{
			CRPS
		}{
			\includegraphics[width=8cm]{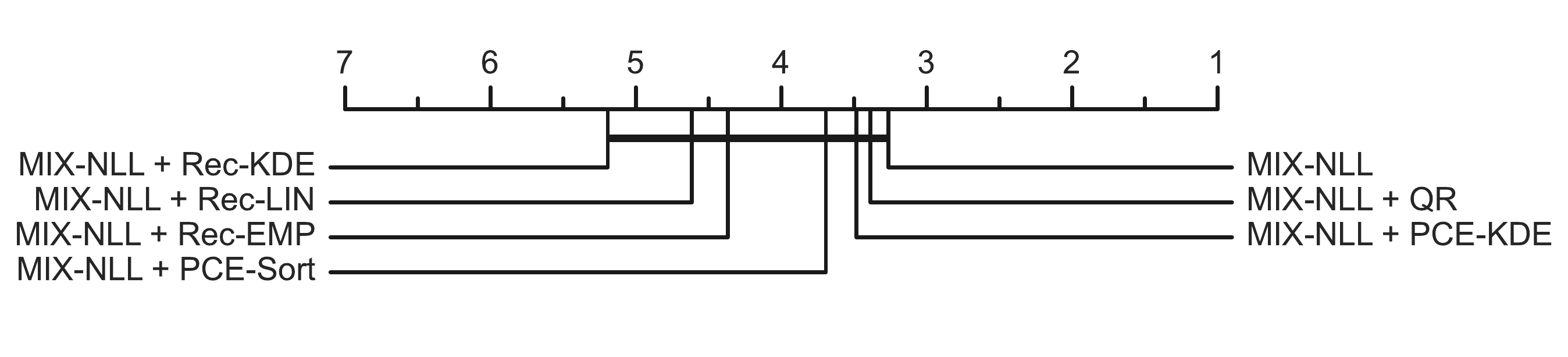}
			\vspace{-0.3cm}
		}
		\subcaptionbox{
			NLL
		}{
			\includegraphics[width=8cm]{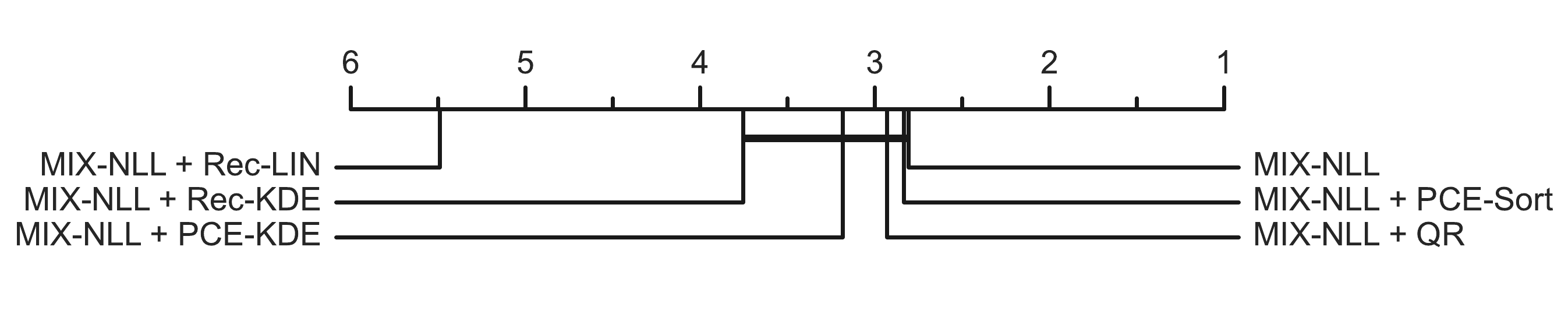}
			\vspace{-0.3cm}
		}
		\subcaptionbox{
			STD
		}{
			\includegraphics[width=8cm]{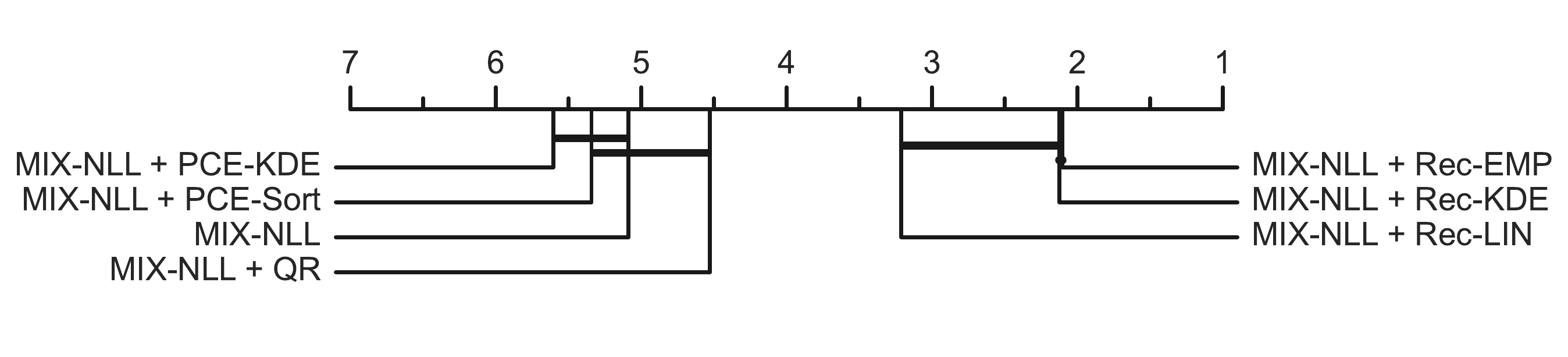}
			\vspace{-0.3cm}
		}
		\caption{Comparison of different metrics where the base model is \texttt{MIX-NLL}.}
		\label{fig:posthoc_or_regul_vs_vanilla/nll}
	\end{figure}
	
	
	\begin{figure}[H]
		\centering
		\subcaptionbox{
			Boxplots of Cohen's d of different metrics on all datasets, with respect to \texttt{MIX-CRPS}.
		}{
			\includegraphics[width=\linewidth]{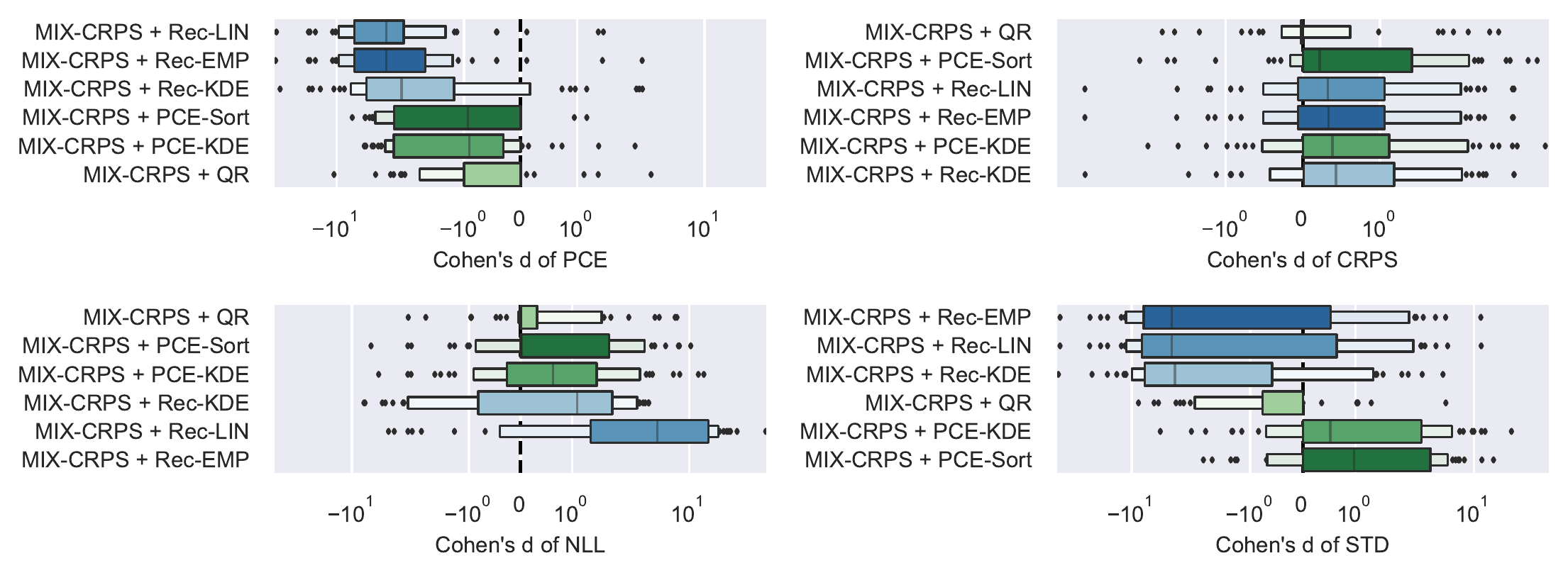}
			\vspace{-0.3cm}
		}
		\subcaptionbox{
			PCE
		}{
			\includegraphics[width=8cm]{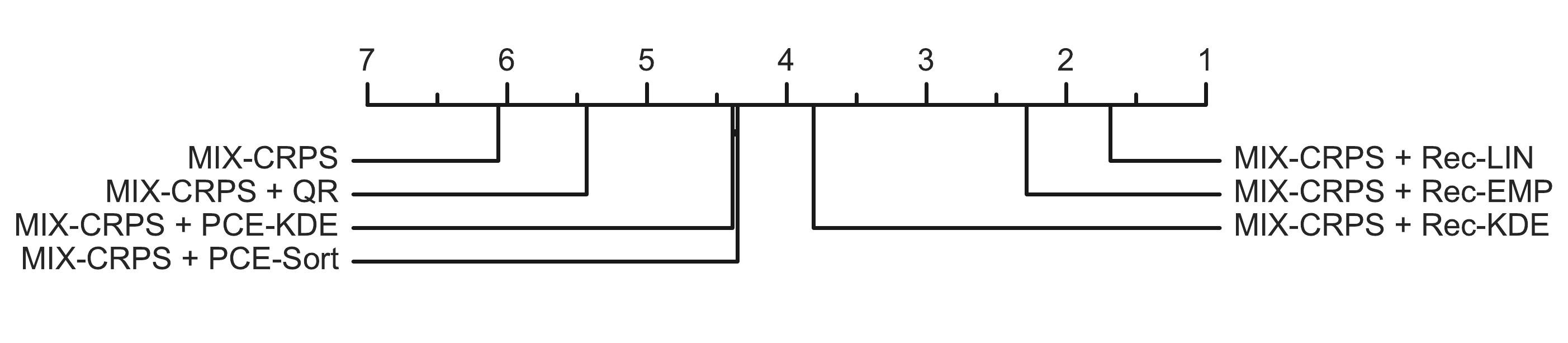}
			\vspace{-0.3cm}
		}
		\subcaptionbox{
			CRPS
		}{
			\includegraphics[width=8cm]{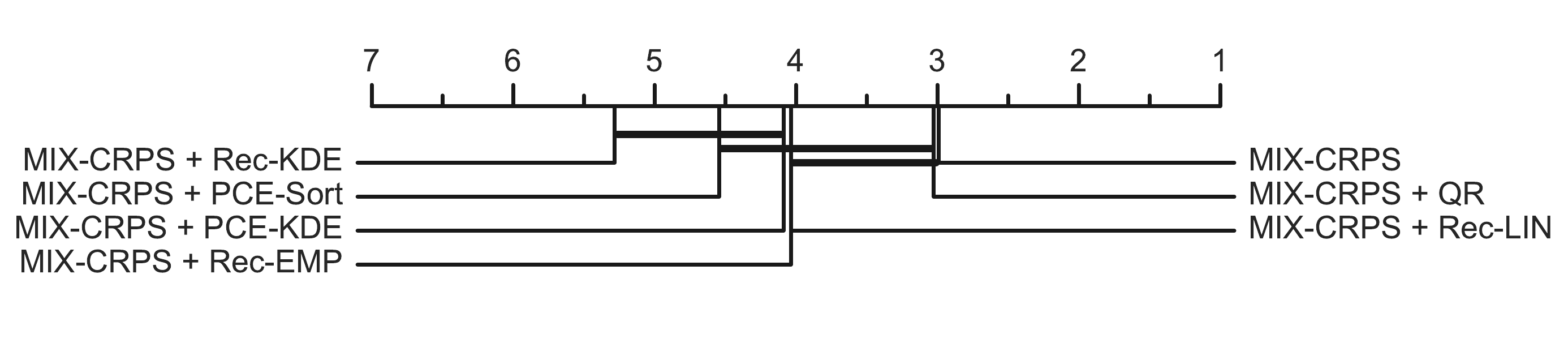}
			\vspace{-0.3cm}
		}
		\subcaptionbox{
			NLL
		}{
			\includegraphics[width=8cm]{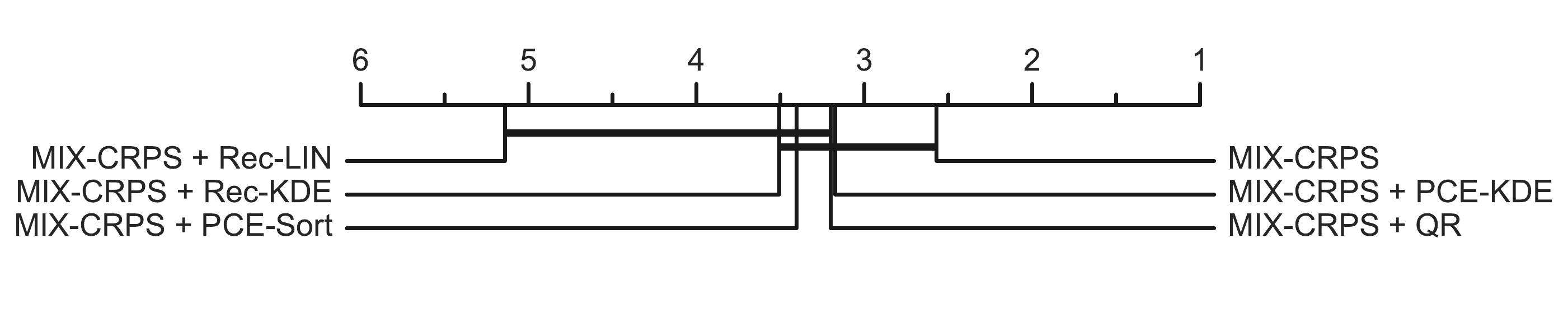}
			\vspace{-0.3cm}
		}
		\subcaptionbox{
			STD
		}{
			\includegraphics[width=8cm]{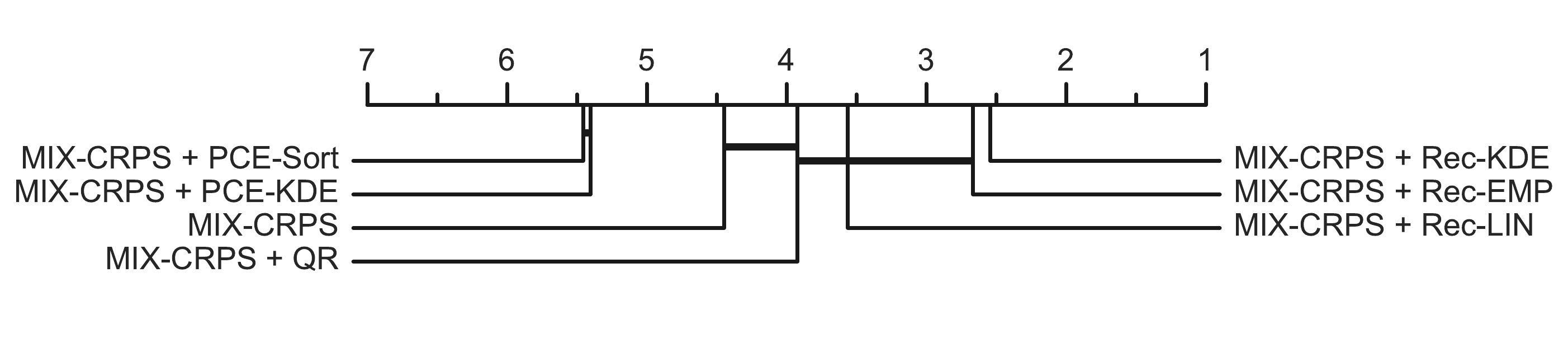}
			\vspace{-0.3cm}
		}
		\caption{Comparison of different metrics where the base model is \texttt{MIX-CRPS}.}
		\label{fig:posthoc_or_regul_vs_vanilla/crps}
	\end{figure}
	
	
	\begin{figure}[H]
		\centering
		\subcaptionbox{
			Boxplots of Cohen's d of different metrics on all datasets, with respect to \texttt{SQR-CRPS}.
		}{
			\includegraphics[width=\linewidth]{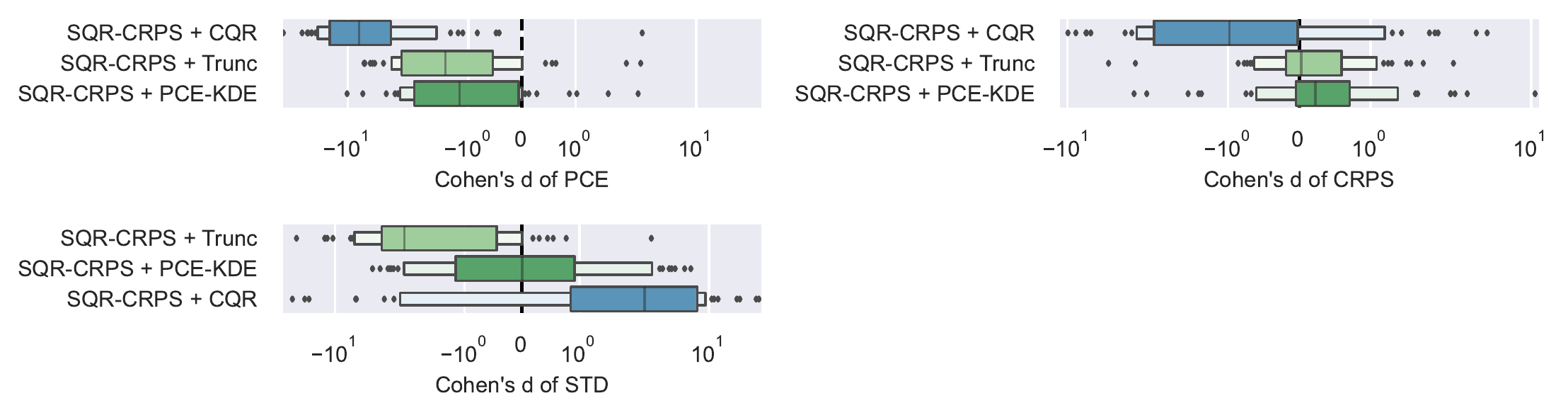}
			\vspace{-0.3cm}
		}
		\subcaptionbox{
			PCE
		}{
			\includegraphics[width=8cm]{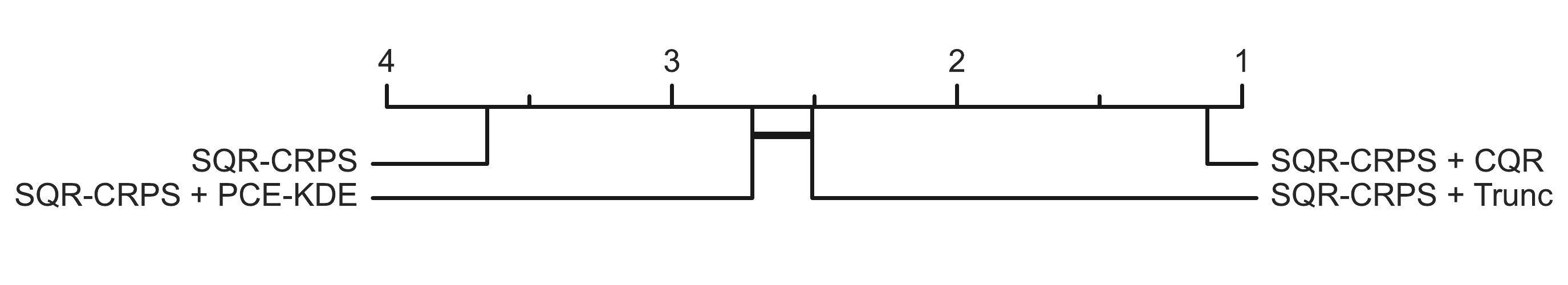}
			\vspace{-0.3cm}
		}
		\subcaptionbox{
			CRPS
		}{
			\includegraphics[width=8cm]{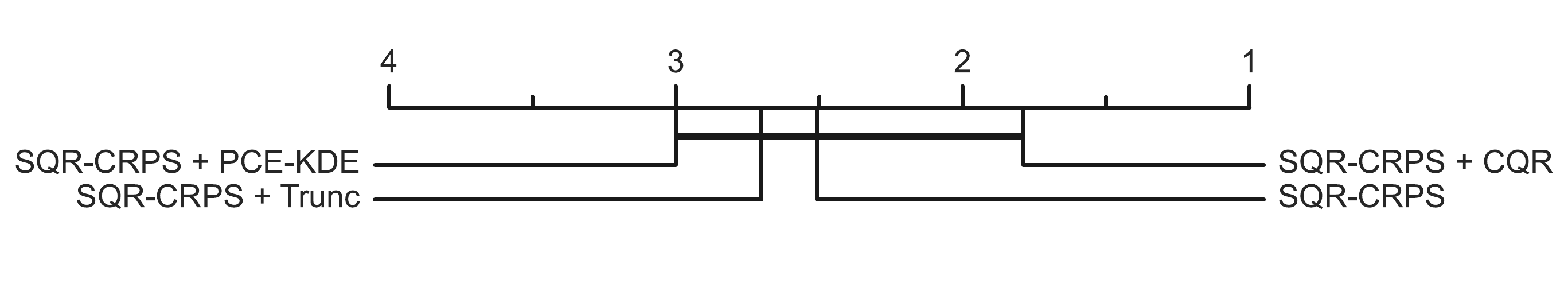}
			\vspace{-0.3cm}
		}
		\subcaptionbox{
			STD
		}{
			\includegraphics[width=8cm]{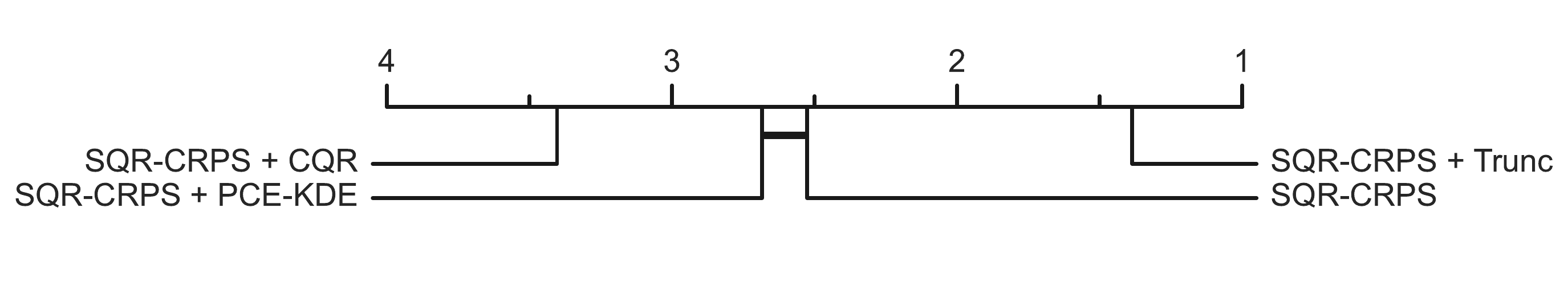}
			\vspace{-0.3cm}
		}
		\hspace{8cm}
		\caption{Comparison of different metrics where the base model is \texttt{SQR-CRPS}.}
		\label{fig:posthoc_or_regul_vs_vanilla/wis}
	\end{figure}
	
	\subsection{Combining Regularization and Post-hoc Methods}
	\label{sec:combining_regul_and_post_hoc}

	In this paper, we have established that post-hoc methods are generally more favorable than regularization methods when the primary objective is to enhance probabilistic calibration. Since regularization methods operate during training and do not alter the form of predictions (e.g., Gaussian mixture predictions), they can be easily combined with post-hoc methods. In this section, we address the question: "Which metrics do regularization methods improve when combined with a post-hoc method compared to the same model without regularization?"

	To ensure clarity, we focus our presentation on a selection of paired regularization and post-hoc methods. \cref{fig:posthoc_and_regul_vs_posthoc} illustrates the impact of regularization on various metrics for these pairs. In \cref{fig:posthoc_and_regul_vs_posthoc/main_metrics/cohen_d_boxplot}, the baseline corresponds to the same post-hoc method without regularization, enabling a direct measurement of the effect of adding regularization to a post-hoc method. It is important to note that the boxplots in this figure cannot be directly compared due to the different baselines.

	The critical difference diagrams provide a comparison of all methods, with and without regularization. Overall, when combined with post-hoc methods, regularization has a negative impact: no regularization method significantly improves probabilistic calibration, and they tend to negatively affect CRPS, NLL, and STD metrics.

	\begin{figure}[H]
		\centering
		\subcaptionbox{
			Boxplots of Cohen's d of different metrics on all datasets, with respect to the same model except that regularization is not applied.
			\label{fig:posthoc_and_regul_vs_posthoc/main_metrics/cohen_d_boxplot}
		}{
			\includegraphics[width=\linewidth]{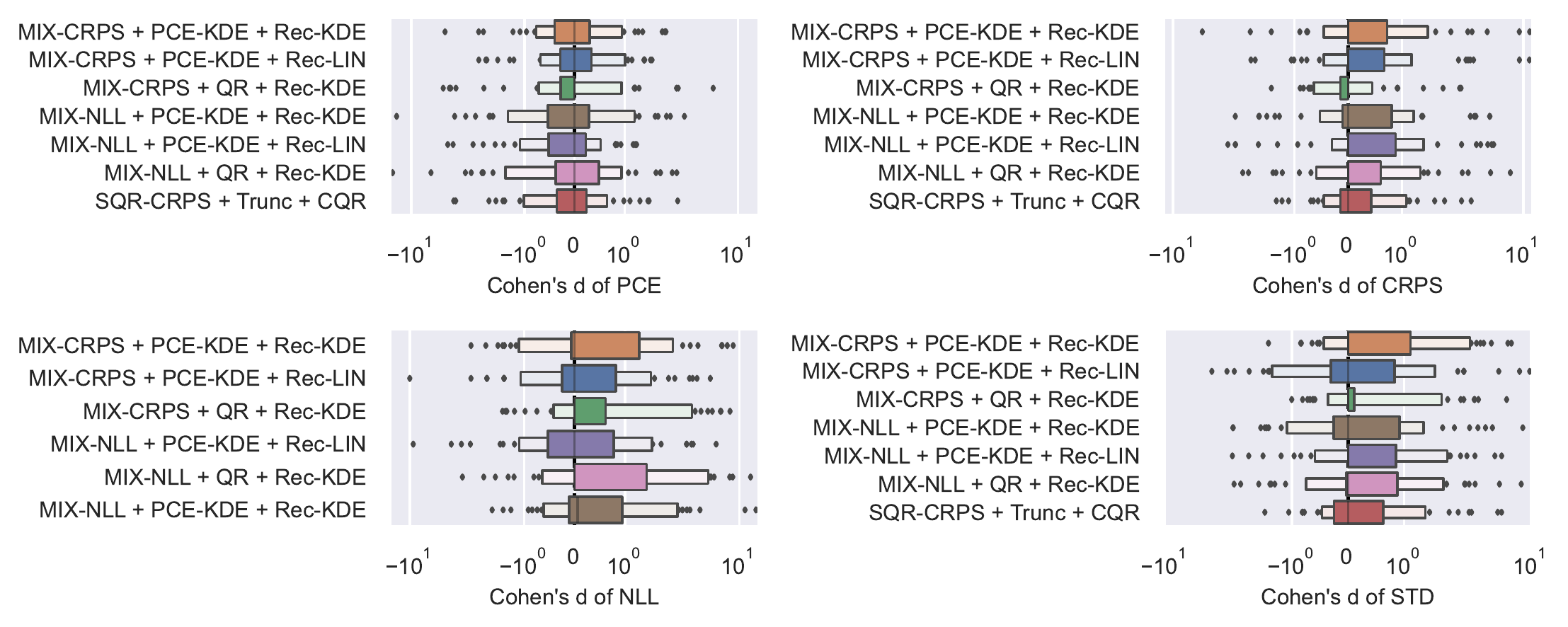}
			\vspace{-0.3cm}
		}
		\subcaptionbox{
			PCE
		}{
			\includegraphics[width=8cm]{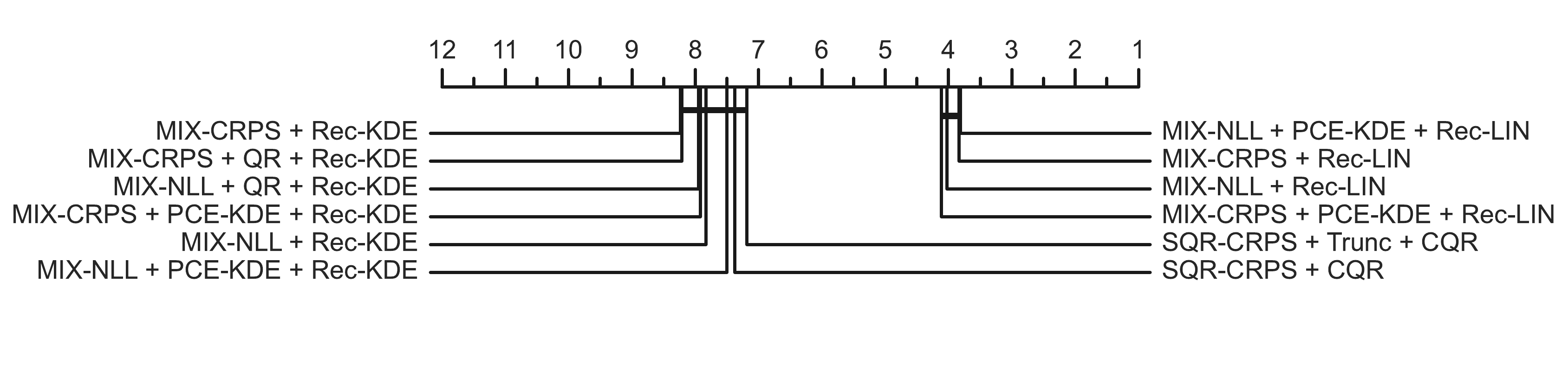}
			\vspace{-0.3cm}
		}
		\subcaptionbox{
			CRPS
		}{
			\includegraphics[width=8cm]{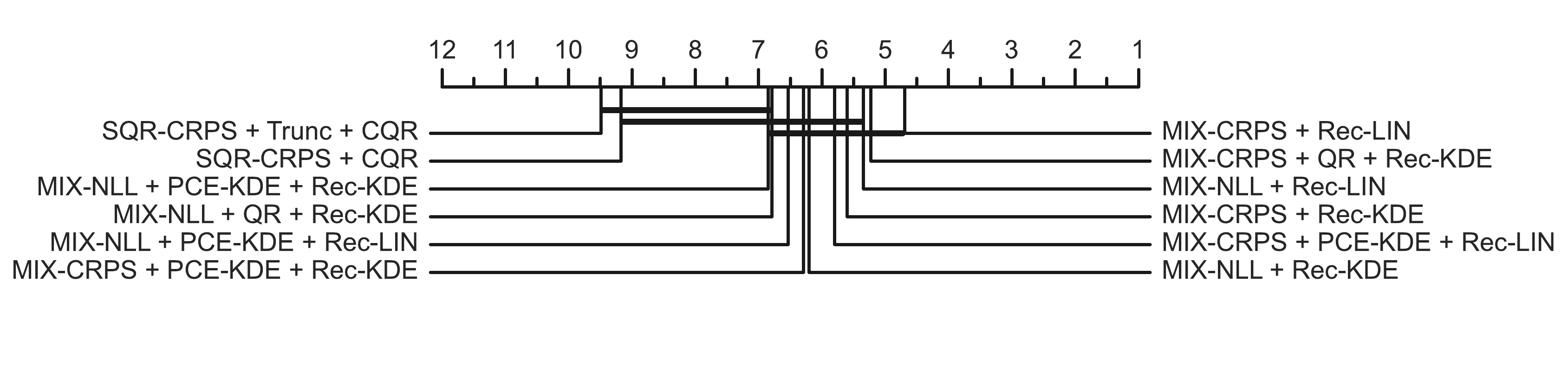}
			\vspace{-0.3cm}
		}
		\subcaptionbox{
			NLL
		}{
			\includegraphics[width=8cm]{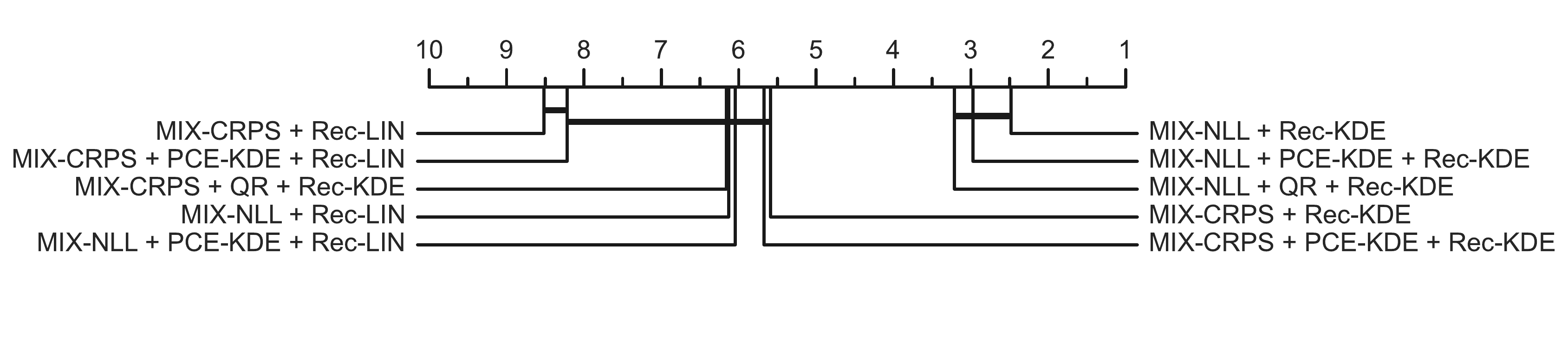}
			\vspace{-0.3cm}
		}
		\subcaptionbox{
			STD
		}{
			\includegraphics[width=8cm]{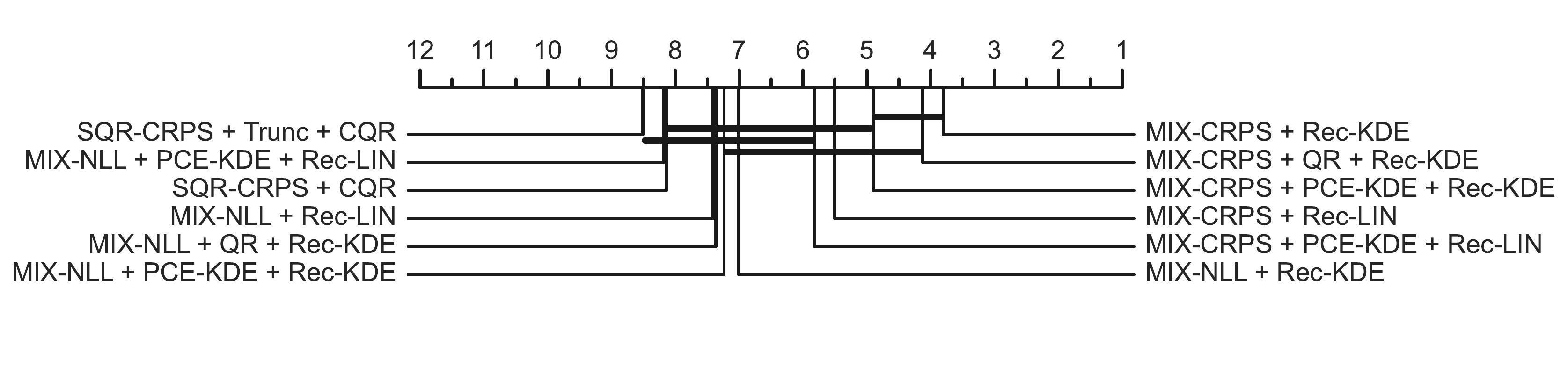}
			\vspace{-0.3cm}
		}
		\caption{Comparison of different metrics showing the effect of regularization when combined with a post-hoc method, compared to the same model without regularization.}
		\label{fig:posthoc_and_regul_vs_posthoc}
	\end{figure}
	
	\subsection{Post-hoc Calibration based on the Training Dataset}
	\label{sec:posthoc_training_dataset}
	
	In this paper, the calibration map or conformity scores have been computed on a separate calibration dataset, following common practice in the literature. However, holding out data for post-hoc calibration reduces the quantity of training data. For the sake of clarity, we focus our analysis on the \texttt{MIX-NLL} and \texttt{SQR-CRPS} base losses
	
	In this section, we compare post-hoc calibration based on the training dataset to post-hoc calibration based on the calibration dataset. We aim to answer the question: "Can it be beneficial to use post-hoc calibration based on the training dataset, and should it be preferred over regularization methods when there is no calibration dataset available?" One advantage of regularization methods and post-hoc calibration methods based on the training dataset is that the base model can be trained on more data (80\% in our experiments, compared to 65\% when holding out the calibration dataset).
	
	\cref{fig:posthoc_dataset_train_vs_calib} presents a comparison of different methods, with post-hoc methods trained on the calibration dataset indicated by \texttt{(calib)} and those trained on the training dataset indicated by \texttt{(train)}. We observe that post-hoc methods based on the calibration dataset tend to significantly outperform their counterparts based on the training dataset in terms of probabilistic calibration. Specifically, \texttt{MIX-NLL + Rec-LIN} and \texttt{MIX-NLL + Rec-KDE} achieve significantly better calibration when the calibration map is learned on the calibration dataset. Similarly, \texttt{SQR-CRPS + CQR} tends to improve calibration when conformal prediction is based on the calibration dataset. It is worth noting that even without a calibration dataset, post-hoc methods tend to be better calibrated than regularization methods.
	
	Finally, we observe that post-hoc methods based on the training dataset tend to achieve better CRPS and NLL scores, although not significantly. Additionally, they are also significantly sharper. This may be attributed to the larger training dataset available to the base model when there is no held-out dataset.

	\begin{figure}[H]
		\centering
		\subcaptionbox{
			Boxplots of Cohen's d of different metrics on all datasets, with respect to \texttt{MIX-NLL}.
		}{
			\includegraphics[width=\linewidth]{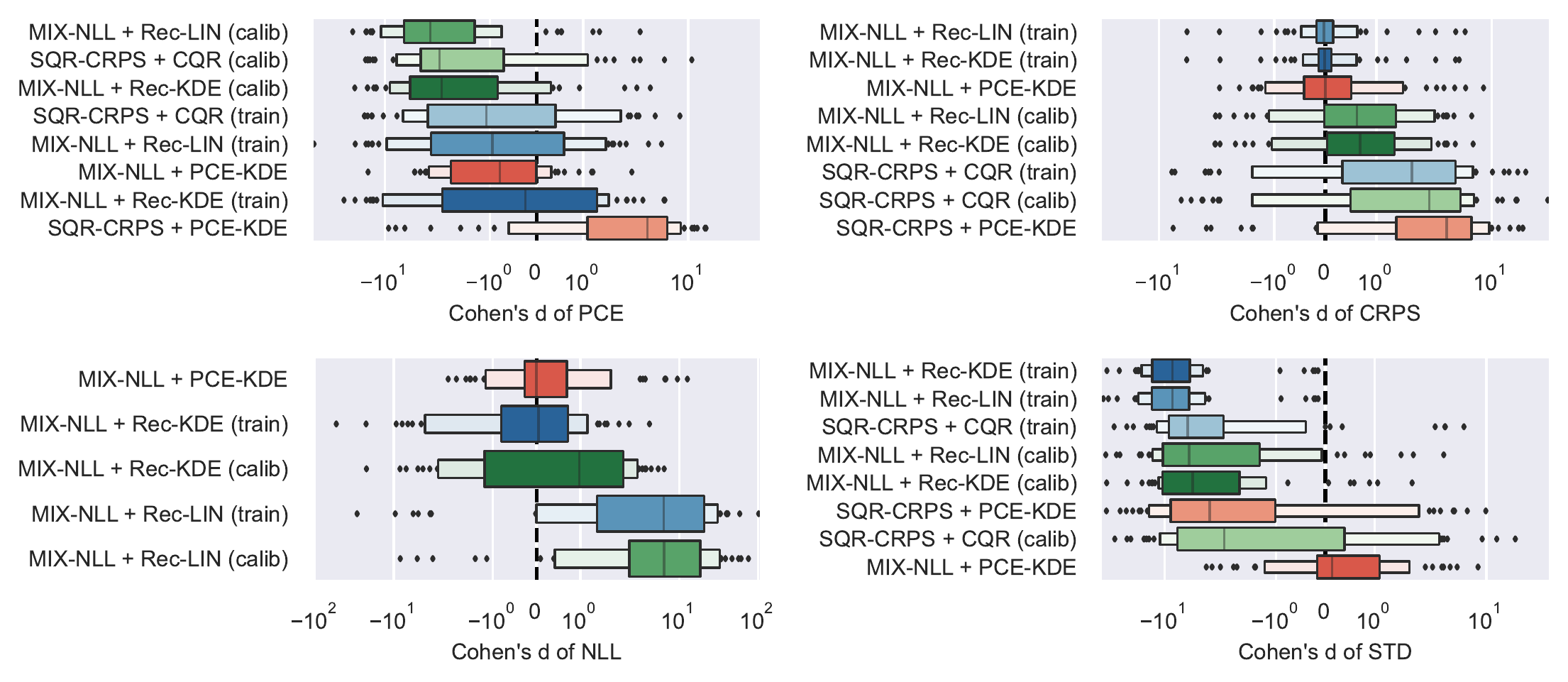}
			\vspace{-0.3cm}
		}
		\subcaptionbox{
			PCE
		}{
			\includegraphics[width=8cm]{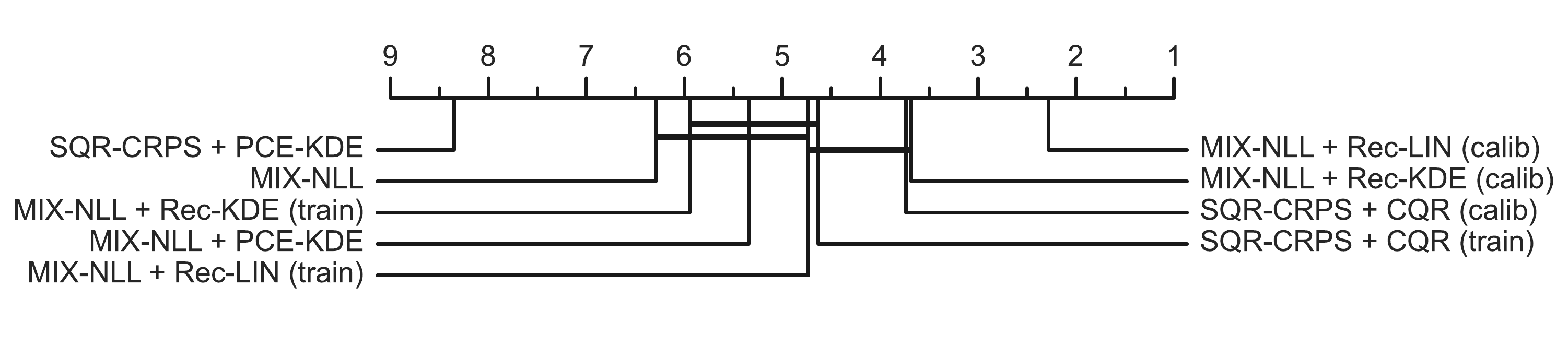}
			\vspace{-0.3cm}
		}
		\subcaptionbox{
			CRPS
		}{
			\includegraphics[width=8cm]{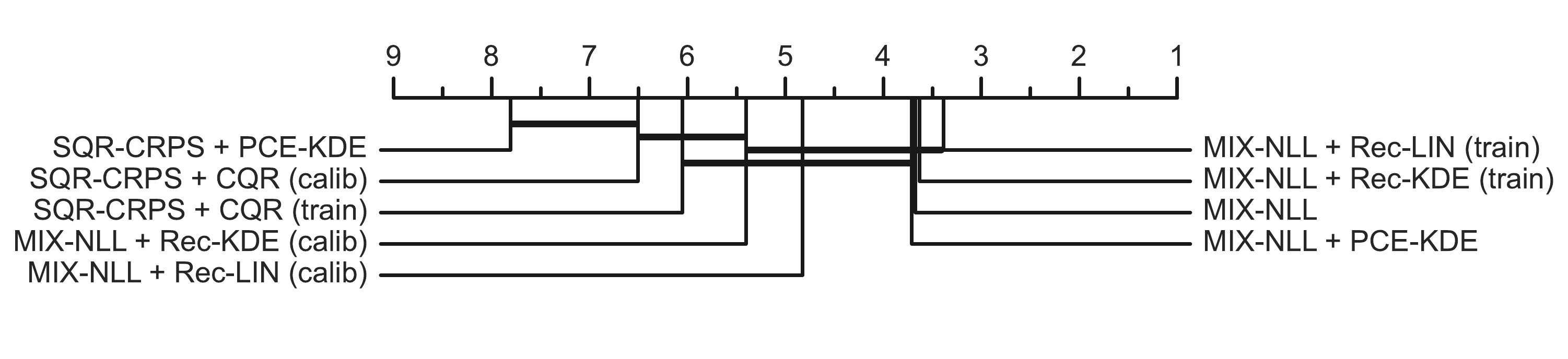}
			\vspace{-0.3cm}
		}
		\subcaptionbox{
			NLL
		}{
			\includegraphics[width=8cm]{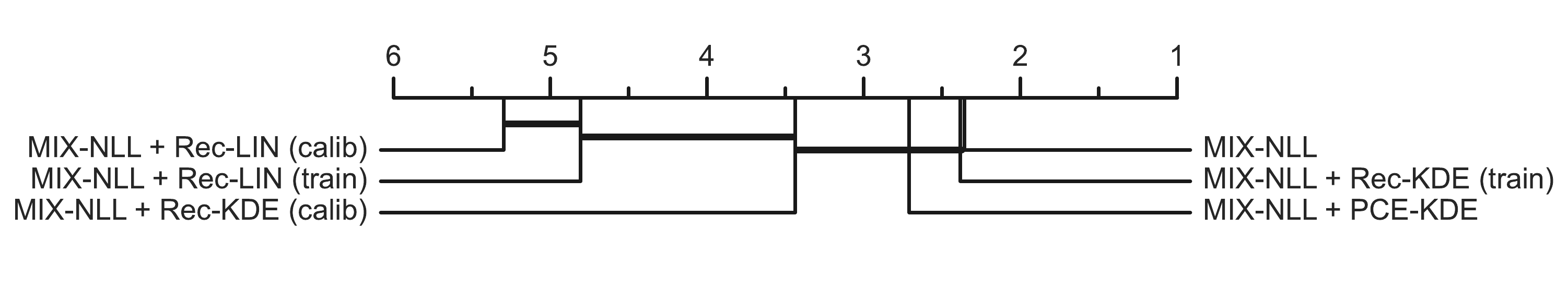}
			\vspace{-0.3cm}
		}
		\subcaptionbox{
			STD
		}{
			\includegraphics[width=8cm]{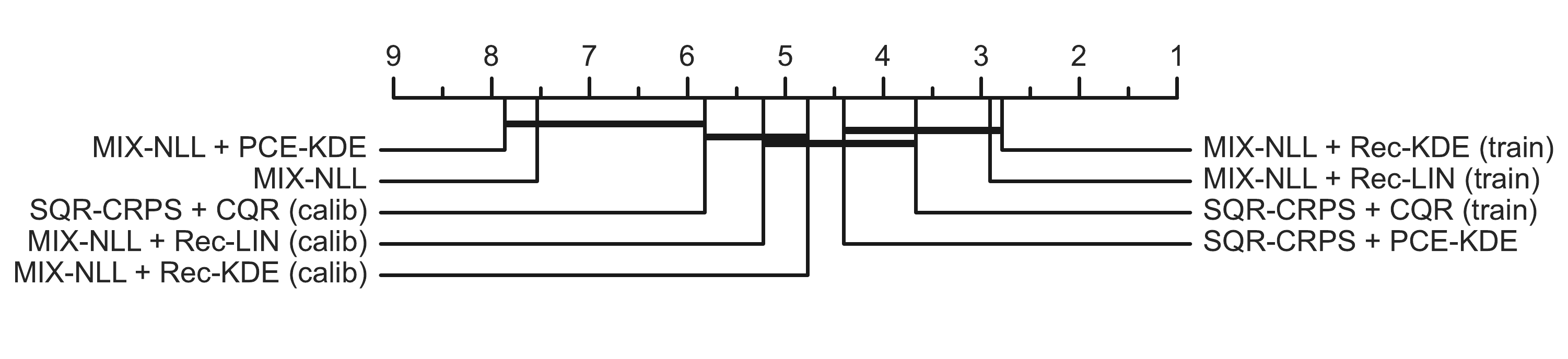}
			\vspace{-0.3cm}
		}
		\caption{Comparison of different metrics.}
		\label{fig:posthoc_dataset_train_vs_calib}
	\end{figure}
	
	%
	
	\subsection{Calibration of Vanilla Models}
	\label{sec:calibration_vanilla}
	
	\cref{fig:pce_and_rel_diags/vanilla_wis} and \cref{fig:pce_and_rel_diags/vanilla_crps} provide additional results from our empirical study in \cref{sec:empirical_study}, specifically focusing on the PCE obtained with \texttt{MIX-CRPS} and \texttt{SQR-CRPS}. The datasets are ordered in the same manner as shown in \cref{fig:uqce_and_rel_diags_p_values} for comparison. We observe that \texttt{SQR-CRPS} is less calibrated compared to \texttt{MIX-NLL}. 
	
	\begin{figure}
		\centering
		\includegraphics[width=\textwidth]{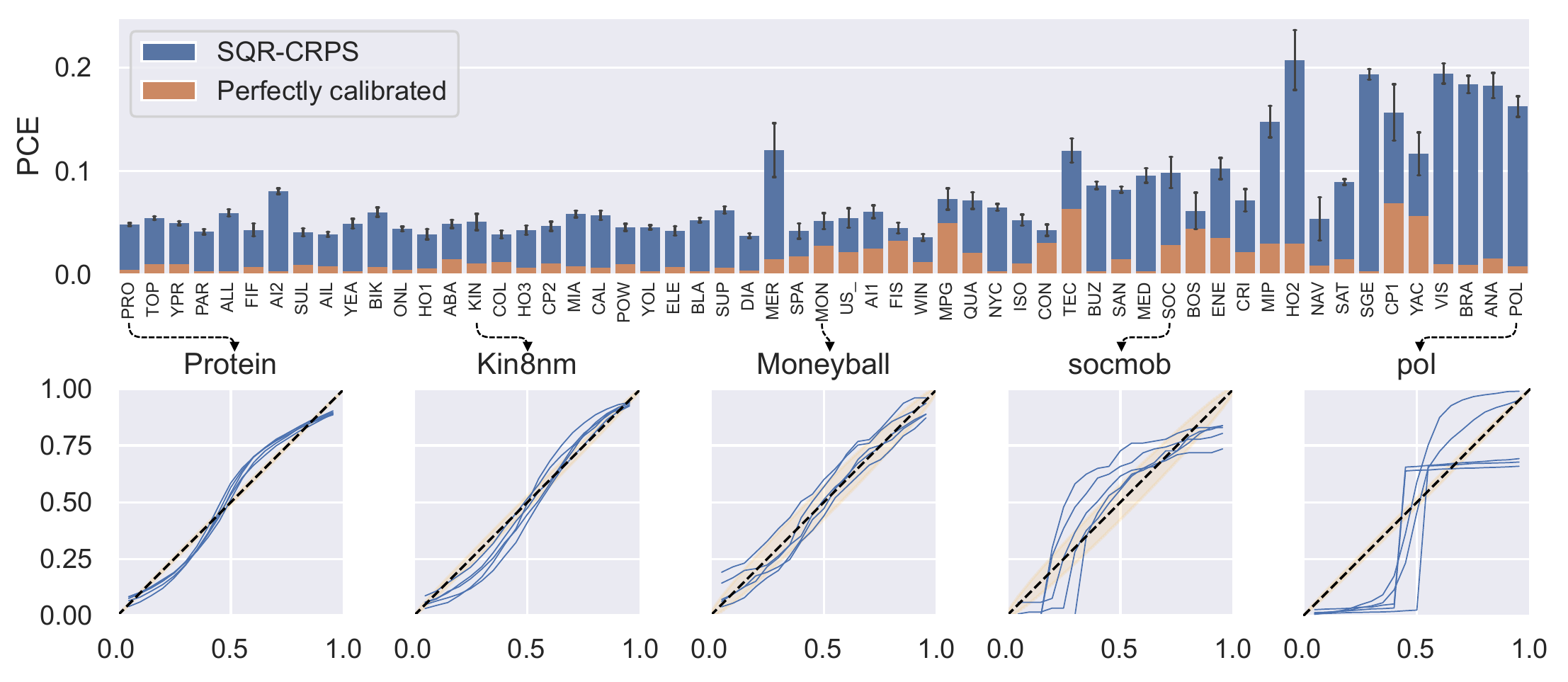}
		\caption{PCE obtained on different datasets, with examples of reliability diagrams.
			The height of each bar is the mean PCE of 5 runs with different dataset splits while the error bar represents the standard error of the mean.
			For 5 datasets, the PIT reliability diagrams of 5 runs are displayed in the bottom row.
		}
		\caption{PCE of \texttt{SQR-CRPS}, on all datasets.}
		\label{fig:pce_and_rel_diags/vanilla_wis}
	\end{figure}
	
	\begin{figure}
		\centering
		\includegraphics[width=\textwidth]{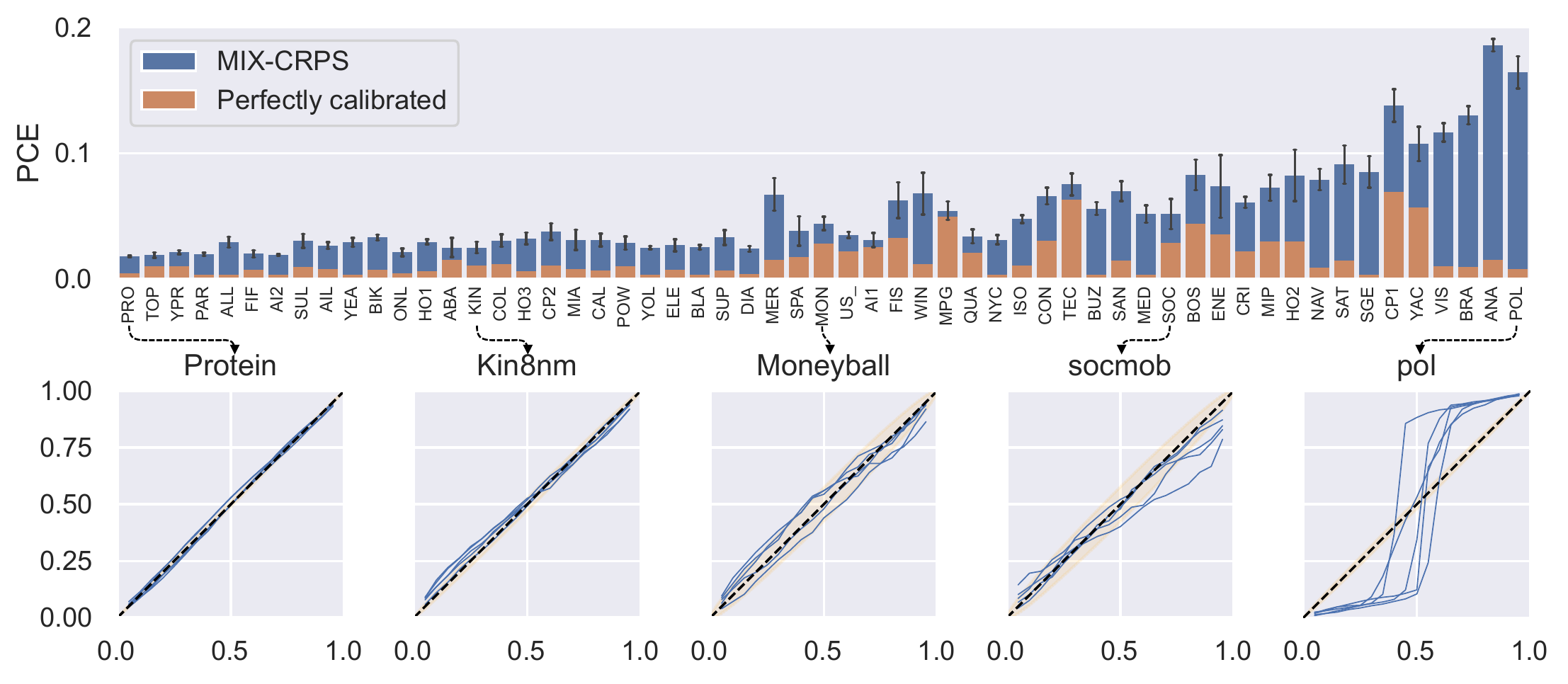}
		\caption{PCE of \texttt{MIX-CRPS}, on all datasets.}
		\label{fig:pce_and_rel_diags/vanilla_crps}
	\end{figure}
	
	\subsection{Distribution of the Test Statistic}
	\label{sec:distribution_test_statistic}
	
	\cref{fig:hist_test_statistic} shows the distribution of the test statistic, as described in \cref{sec:empirical_study}.
	\rev{We observe that, in a lot of cases, the average PCE of the compared models is larger than all the $10^4$ samples of the average PCE from a probabilistically calibrated model.
		Among the different calibration methods, post-hoc calibration with \texttt{MIX-NLL + Rec-EMP} achieves the highest level of calibration performance in the majority of cases.}
	
	\begin{figure}[H]
		\centering
		\includegraphics[width=\linewidth]{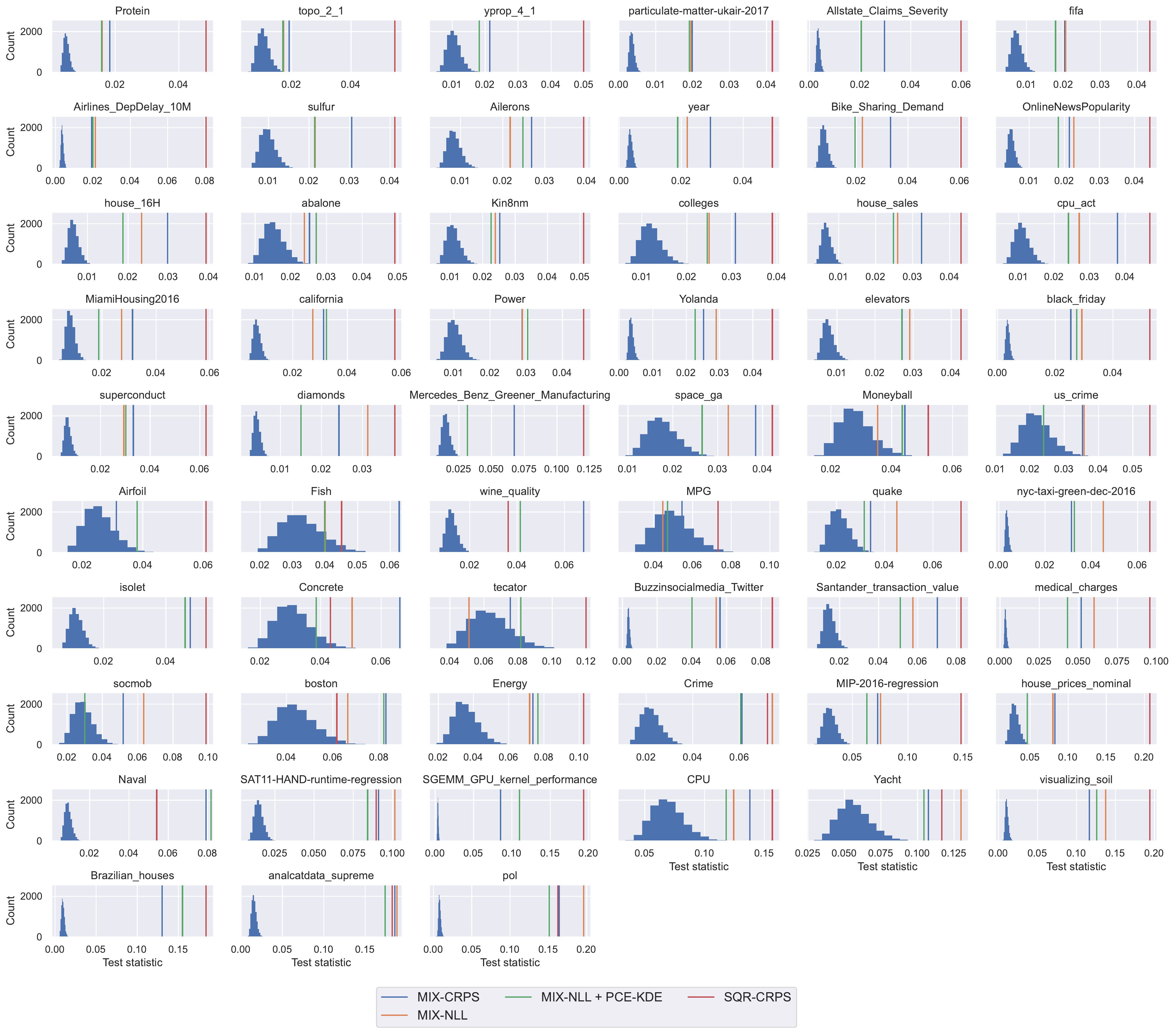}
		\caption{Distribution of the test statistic on all datasets for different models.}
		\label{fig:hist_test_statistic}
	\end{figure}
	
	\subsection{Reliability Diagrams}
	\label{sec:rel_diags}
	
	\cref{fig:reliability_base_loss} and \cref{fig:reliability_posthoc} compare reliability diagrams obtained on models with and without post-hoc calibration, respectively. With only a few exceptions, the post-hoc calibrated models exhibit a visual proximity to the diagonal line.

	\begin{figure}[H]
		\centering
		\includegraphics[width=0.95\linewidth]{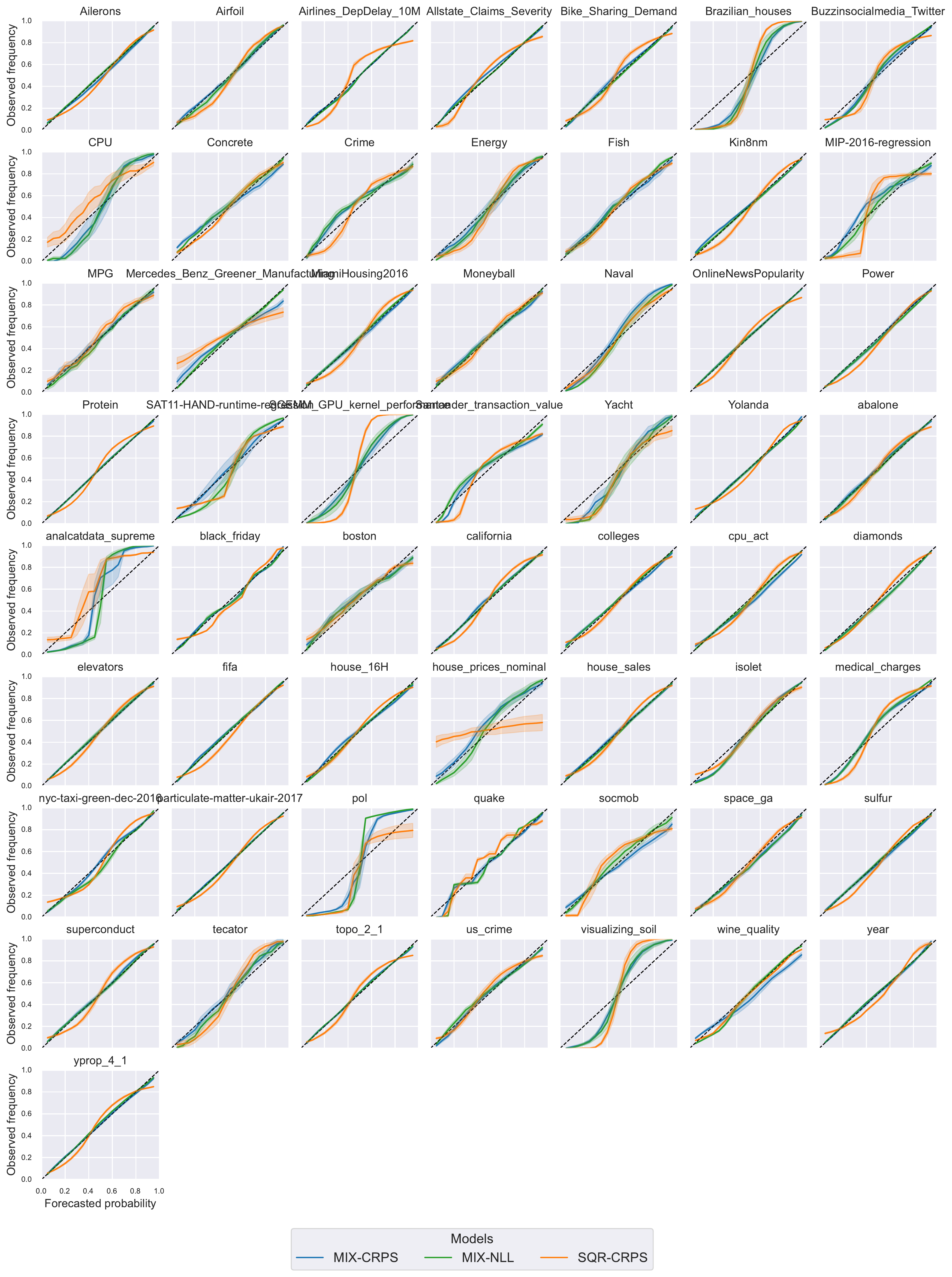}
		\caption{Reliability diagrams on all datasets for different models.}
		\label{fig:reliability_base_loss}
	\end{figure}
	
	\begin{figure}[H]
		\centering
		\includegraphics[width=0.95\linewidth]{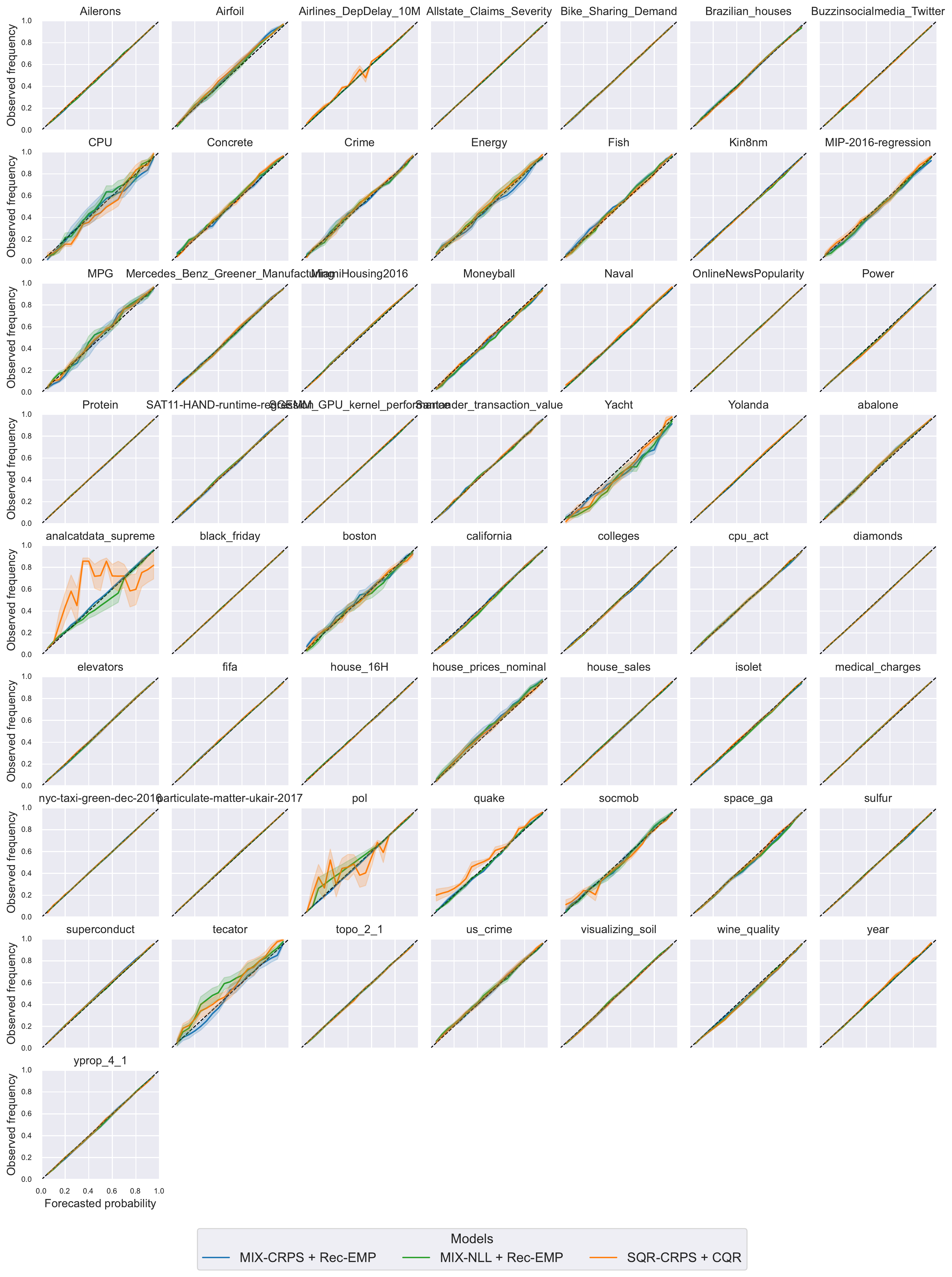}
		\caption{Reliability diagrams on all datasets for different models with post-hoc calibration.}
		\label{fig:reliability_posthoc}
	\end{figure}
	
	\section{Hyperparameters}
	\label{sec:hparams}
	
	In our experiments, we adopt a specific architecture consisting of 3 hidden layers with 100 units per layer, ReLU non-linearities, and a dropout rate of 0.2 on the last hidden layer. Early stopping with a patience of 30 is applied to select the epoch with the lowest base loss on the validation dataset.
	
	In this section, we delve into the performance of different model parameters, including the number of components in Gaussian mixture predictions, the number of quantiles in quantile predictions, and the number of hidden layers in the underlying models.

	\cref{fig:hparams/mixture_size} compares models that predict mixtures with varying numbers of components compared to the reference of 3 components. Notably, when there is only 1 component (yielding a single Gaussian prediction), the model's performance significantly deteriorates in terms of CRPS, NLL, and sharpness. However, as the number of components increases beyond 3, the differences become less pronounced.
	
	\cref{fig:hparams/n_quantiles} compares models with different numbers of quantiles compared to a reference of 64 quantiles. The results reveal a consistent pattern: predicting more quantiles consistently enhances performance in terms of probabilistic calibration, CRPS, and sharpness.
	
	\cref{fig:hparams/nb_hidden} compares models with different numbers of layers relative to a 3-layer model. It highlights that models with 2, 3, or 5 layers tend to yield superior performance in terms of CRPS and NLL.
	
	\begin{figure}[H]
		\centering
		\includegraphics[width=\linewidth]{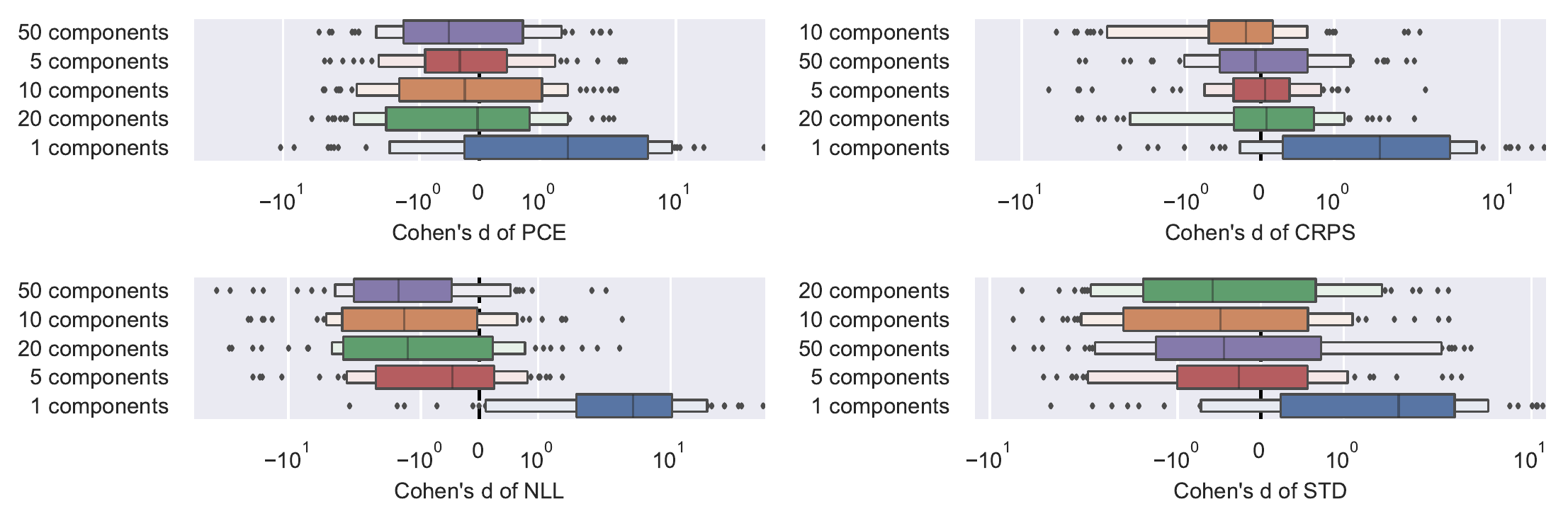}
		\caption{Comparison of models whose predictions are Gaussian mixtures with different numbers of components.
			All models are trained with NLL loss, without regularization or post-hoc method.
			The box plots show Cohen's d of different metrics on all datasets.
			Cohen's d is computed with respect to a model whose predictions are Gaussian mixtures with 3 components.}
		\label{fig:hparams/mixture_size}
	\end{figure}
	
	\begin{figure}[H]
		\centering
		\includegraphics[width=\linewidth]{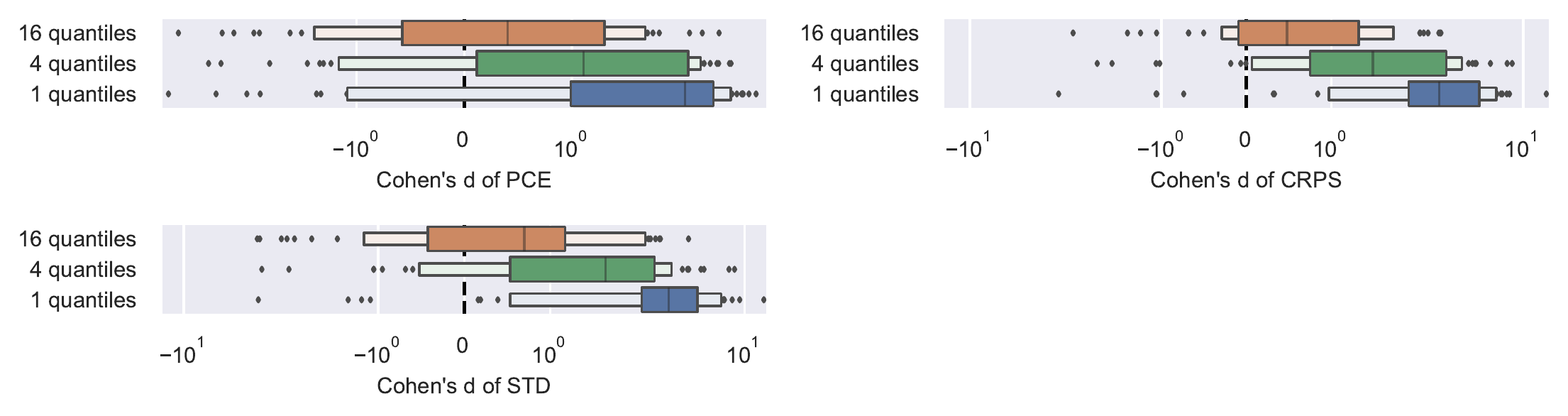}
		\caption{Comparison of models whose predictions are different numbers of quantiles.
			All models are trained with CRPS loss, without regularization or post-hoc method.
			The box plots show Cohen's d of different metrics on all datasets.
			Cohen's d is computed with respect to a model whose predictions are 64 quantiles.}
		\label{fig:hparams/n_quantiles}
	\end{figure}
	
	\begin{figure}[H]
		\centering
		\includegraphics[width=\linewidth]{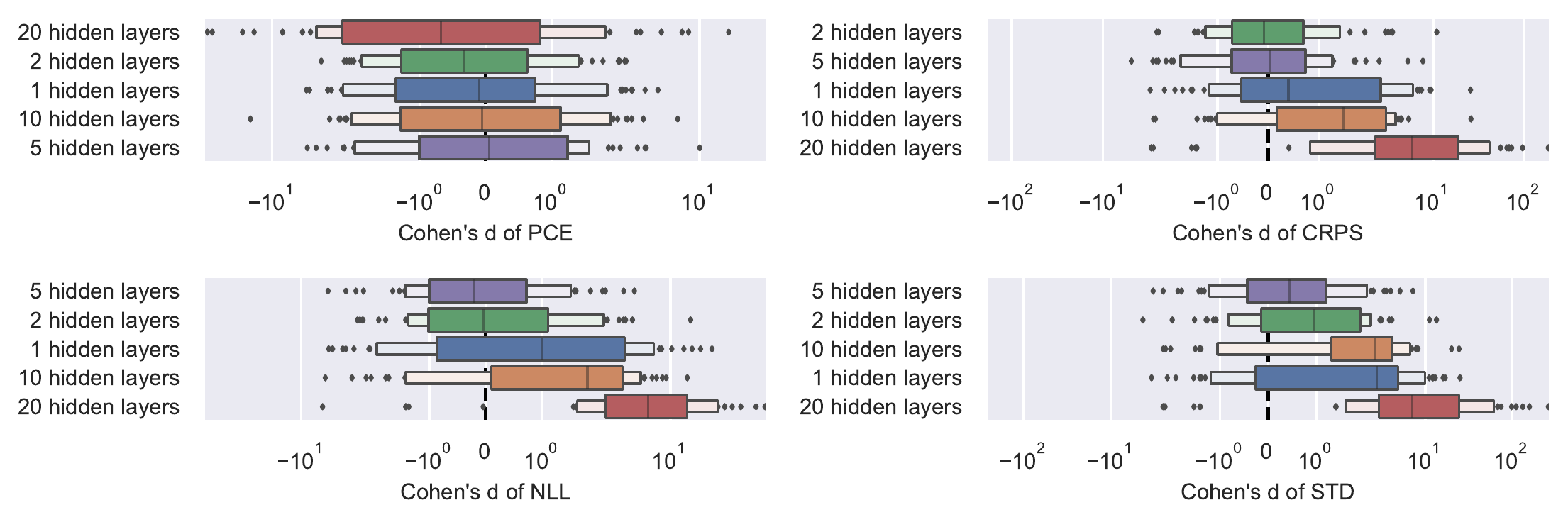}
		\caption{Comparison of models with different number of layers.
			All models predict Gaussian mixtures and are trained with NLL loss, without regularization or post-hoc method.
			The box plots show Cohen's d of different metrics on all datasets.
			Cohen's d is computed with respect to a model with 3 hidden layers.}
		\label{fig:hparams/nb_hidden}
	\end{figure}

	\section{Tabular Regression Datasets}
	\label{sec:appendix_datasets}

	\cref{table:datasets}  presents the datasets considered in our experiments. To ensure consistency, when datasets are available from multiple sources, we select one specific source per dataset, as indicated in \cref{fig:papers_vs_datasets}. Our selection prioritizes the suites 297, 299, and 269 of OpenML, followed by UCI datasets.
	
	In the OpenML suite 297, we discovered that the datasets \texttt{houses} and \texttt{california} are identical, and thus, we only included the \texttt{california} dataset in our analysis. Moreover, the UCI archive for the dataset \texttt{wine\_quality} contains two separate datasets for red and white wine. As there was no indication regarding the specific dataset(s) used in previous studies, we followed the approach of \citet{Grinsztajn2022-nu} and solely considered the dataset related to white wine. In \cref{fig:papers_vs_datasets}, other studies may have employed the alternative dataset or a combination of both datasets.

	\begin{table}[H]
		\fontsize{8.5pt}{9.5pt}
		\selectfont	
		\caption{Datasets}
		\label{table:datasets}
		\centering
		\begin{tabular}{lllrr}
	\toprule
	&  &  & \makecell{Nb of \\ training instances} & \makecell{Nb of \\ features} \\
	Group & Dataset & Abbrev. &  & \\
	\midrule
	\multirow[t]{12}{*}{UCI} & CPU & CP1 & 135 & 7 \\
	& Yacht & YAC & 200 & 6 \\
	& MPG & MPG & 254 & 7 \\
	& Energy & ENE & 499 & 9 \\
	& Crime & CRI & 531 & 104 \\
	& Fish & FIS & 590 & 6 \\
	& Concrete & CON & 669 & 8 \\
	& Airfoil & AI1 & 976 & 5 \\
	& Kin8nm & KIN & 5324 & 8 \\
	& Power & POW & 6219 & 4 \\
	& Naval & NAV & 7757 & 17 \\
	& Protein & PRO & 29724 & 9 \\
	\multirow[t]{19}{*}{OpenML 297} & wine\_quality & WIN & 4223 & 11 \\
	& isolet & ISO & 5068 & 613 \\
	& cpu\_act & CP2 & 5324 & 21 \\
	& sulfur & SUL & 6552 & 6 \\
	& Brazilian\_houses & BRA & 6949 & 8 \\
	& Ailerons & AIL & 8937 & 33 \\
	& MiamiHousing2016 & MIA & 9055 & 13 \\
	& pol & POL & 9750 & 26 \\
	& elevators & ELE & 10789 & 16 \\
	& Bike\_Sharing\_Demand & BIK & 11296 & 6 \\
	& fifa & FIF & 11740 & 5 \\
	& california & CAL & 13416 & 8 \\
	& superconduct & SUP & 13820 & 79 \\
	& house\_sales & HO3 & 14048 & 15 \\
	& house\_16H & HO1 & 14809 & 16 \\
	& diamonds & DIA & 35061 & 6 \\
	& medical\_charges & MED & 50000 & 3 \\
	& year & YEA & 50000 & 90 \\
	& nyc-taxi-green-dec-2016 & NYC & 50000 & 9 \\
	\multirow[t]{8}{*}{OpenML 299} & analcatdata\_supreme & ANA & 2633 & 12 \\
	& \makecell{Mercedes\_Benz\\\_Greener\_Manufacturing} & MER & 2735 & 735 \\
	& visualizing\_soil & VIS & 5616 & 5 \\
	& yprop\_4\_1 & YPR & 5775 & 82 \\
	& OnlineNewsPopularity & ONL & 25768 & 73 \\
	& black\_friday & BLA & 50000 & 23 \\
	& \makecell{SGEMM\_GPU\\\_kernel\_performance} & SGE & 50000 & 15 \\
	& \makecell{particulate-matter\\-ukair-2017} & PAR & 50000 & 26 \\
	\multirow[t]{18}{*}{OpenML 269} & tecator & TEC & 156 & 124 \\
	& boston & BOS & 328 & 22 \\
	& MIP-2016-regression & MIP & 708 & 111 \\
	& socmob & SOC & 751 & 39 \\
	& Moneyball & MON & 800 & 18 \\
	& house\_prices\_nominal & HO2 & 711 & 234 \\
	& us\_crime & US\_ & 1295 & 101 \\
	& quake & QUA & 1415 & 3 \\
	& space\_ga & SPA & 2019 & 6 \\
	& abalone & ABA & 2715 & 10 \\
	& \makecell{SAT11-HAND-\\runtime-regression} & SAT & 2886 & 118 \\
	& \makecell{Santander\_transaction\\\_value} & SAN & 2898 & 3611 \\
	& colleges & COL & 4351 & 34 \\
	& topo\_2\_1 & TOP & 5775 & 252 \\
	& Allstate\_Claims\_Severity & ALL & 50000 & 477 \\
	& Yolanda & YOL & 50000 & 100 \\
	& Buzzinsocialmedia\_Twitter & BUZ & 50000 & 70 \\
	& Airlines\_DepDelay\_10M & AI2 & 50000 & 5 \\
	\bottomrule
\end{tabular}

	\end{table}
	
	\let\clearpage\relax
	
\end{document}